%% file: main.tex
\documentclass{article}
\usepackage{iclr2027_conference,times}
\usepackage{amsmath,amssymb,amsthm}
\usepackage{graphicx}
\usepackage{booktabs,tabularx,array}
\usepackage{xcolor}
\usepackage{tikz}
\usetikzlibrary{arrows.meta}
\usepackage{hyperref}
\usepackage{url}

\input{macros}

\hypersetup{
 hidelinks,
 pdftitle={When to Stop Penalizing Actions? Martingale GRPO for VLA Reinforcement Learning},
 pdfauthor={Zongyao Zhao, Cheng-Jie Zhao, Ruihang Jiang, Bing Zheng, Zhaolin Wang, Yuanwei Liu}
}

\title{When to Stop Penalizing Actions? \\ Martingale GRPO for \\ VLA Reinforcement Learning}
\author{\small\textbf{Zongyao Zhao\textsuperscript{1}, Cheng-Jie Zhao\textsuperscript{1}, Ruihang Jiang\textsuperscript{1}, Bing Zheng\textsuperscript{2}, Zhaolin Wang\textsuperscript{1}, Yuanwei Liu\textsuperscript{1}} \\
\textsuperscript{1}The University of Hong Kong \\
\textsuperscript{2}Tsinghua University}
\iclrfinalcopy

\begin{document}
\maketitle
\lhead{Preprint}
\suppressfloats[t]

\begin{abstract}
\input{sections/abstract}
\end{abstract}

\input{sections/introduction}
\input{sections/problem_setup}
\input{sections/method}
\input{sections/experiments}

\input{sections/conclusion}
\label{maintext:end}

\clearpage
\input{sections/statements}
\bibliography{references}
\bibliographystyle{iclr2027_conference}

\clearpage
\appendix
\input{appendix/contents}
\input{appendix/related_work}
\input{appendix/appendix}

\end{document}

%% file: macros.tex
\newcommand{\ind}{\mathbf{1}}
\newcolumntype{Y}{>{\raggedright\arraybackslash}X}

\newcommand{\E}{\mathbb{E}}
\newcommand{\F}{\mathcal{F}}
\newcommand{\Var}{\operatorname{Var}}
\newcommand{\Cov}{\operatorname{Cov}}
\newcommand{\tr}{\operatorname{tr}}
\newcommand{\martgrpo}{\mbox{\textsc{MART-GRPO}}}
\newtheorem{theorem}{Theorem}
\newtheorem{proposition}[theorem]{Proposition}
\newtheorem{corollary}[theorem]{Corollary}

\theoremstyle{definition}

%% file: sections/abstract.tex
A failed vision-language-action (VLA) rollout does not imply that every action contributed to failure.
Yet in group-relative policy optimization (GRPO), a single trajectory-level advantage is assigned to all actions in a rollout.
Potentially useful recovery attempts can thus be penalized alongside the errors they aim to correct.
Ideal stopping conditions for failure feedback are characterized using success-probability and policy-score martingales.
At these stops, the removed gradient contribution has zero conditional expectation given the retained information.
Truncation therefore preserves the expectation of a reference GRPO group-gradient estimator and can reduce its covariance.
One such stop occurs when the conditional probability of success is zero, motivating error and recovery cues for selecting candidate feedback boundaries.
Martingale GRPO (\martgrpo) is introduced with a vision-language model (VLM) for event-guided truncation.
The VLM is fine-tuned on our annotated training set using low-rank adaptation (LoRA).
The last visible error without observed recovery is localized by the frozen VLM in complete rollout videos.
Failure feedback is retained through the selected error interval and removed afterward, while successful-rollout loss terms and original loss normalization are preserved.
For retrospective boundaries, changes in expectation and variance are analyzed to assess departures from ideal stopping.
Experiments with OpenVLA-OFT on RoboTwin 2.0 show that \martgrpo\ improves success rates over full-trajectory GRPO across multiple tasks.

%% file: sections/introduction.tex
\section{Introduction}
\label{sec:introduction}

Vision-language-action (VLA) models map visual observations and language instructions to actions~\citep{zitkovich2023rt2,kim2025openvla}. Pretrained policies are adapted through supervised fine-tuning on expert demonstrations~\citep{kim2025oft} and further optimized through reinforcement learning from environment interaction~\citep{lu2025vlarl,liu2025rlvla}. Group-relative policy optimization (GRPO)~\citep{shao2024deepseekmath} turns terminal task rewards into trajectory-level advantages by comparing rollouts collected for the same task. Each advantage is then applied across the rollout's action chunks.

A terminal failure does not imply that every action contributed to failure. For example, a policy may drop an object, attempt to regrasp it, and exhaust its interaction budget (Figure~\ref{fig:overview}). In a mixed-outcome group, the failed rollout's negative advantage penalizes both the error and the subsequent recovery attempts. This motivates the question: \emph{When to stop penalizing actions?}

\input{figures/overview}

Several methods already change which actions receive feedback, each at a different point in the pipeline. Temporal GRPO reconstructs stage-relative advantages but requires task-stage detection and rollout alignment~\citep{zhou2026temporal}. PCM allocates gradient computation through chunk sampling guided by phase-level success--failure divergence~\citep{bagaria2026pcm}. RedFlow's corrective targets require suitable successful matches in offline data~\citep{yan2026redflow}. In LLM reinforcement learning, negative-advantage updates are masked in $\mu$-GRPO using importance-ratio triggers~\citep{tian2026mugrpo}. ESPO terminates rollouts early~\citep{li2026espo}, whereas PS-PPO samples update prefixes with importance weighting to preserve the reference gradient expectation~\citep{hwang2026psppo}. Here, failed-rollout suffixes are masked using visible errors and recovery cues after complete execution. Group advantages and original loss normalization are retained without additional correction weights for truncation.

To identify a principled truncation boundary, we ask when future actions make no expected gradient contribution. At a fixed policy, the conditional success probability and the cumulative policy score are both martingales, and their increments carry that contribution. Together they define the \emph{remaining conditional credit}: the expected policy-gradient contribution of future actions given the observed history. For example, when failure is already determined by the observed history, the remaining credit is zero. If the decision to stop reads only the observed prefix and the remaining credit is zero, the removed failure contribution has zero conditional expectation. For conditionally independent rollouts whose normalization factors are measurable with respect to the retained information, the truncated GRPO group-gradient estimator is a conditional projection of the full reference estimator. It therefore has the same expectation and no greater covariance.

Motivated by this analysis, we propose \emph{Martingale GRPO} (\martgrpo) for event-guided failure-feedback selection. A vision-language model (VLM) selects feedback boundaries based on visible errors and whether recovery follows. The VLM is fine-tuned on error and recovery annotations and frozen during policy optimization. It reads each complete failed rollout and localizes the onset of the last visible error without observed recovery to a chunk interval. Later frames help interpret earlier events, so the boundary is selected retrospectively. The interval's right endpoint defines $B$: failure feedback is retained from chunk $0$ through $B$ and removed afterward. Successful-rollout loss terms are retained. A conditional residual relates these retrospective boundaries to the conditional projection above and bounds their effects on gradient expectation and total variance.

Our contributions are threefold.
\begin{itemize}
 \makeatletter\@beginparpenalty=10000\makeatother
 \item \textbf{A martingale analysis of failure feedback.} Ideal stopping rules are characterized through remaining conditional credit and a conditional projection of the complete GRPO group-gradient estimator. The projection preserves the reference expectation without increasing covariance, while residual and bias--variance analyses characterize the effects of retrospective boundaries.
 \item \textbf{Event-guided feedback truncation with \martgrpo.} A VLM is fine-tuned with LoRA on our annotated training set and frozen to localize the last visible error without observed recovery in complete rollouts. These event intervals are converted into failure-prefix masks that preserve successful-rollout loss terms, group advantages, and original loss normalization.
 \item \textbf{Empirical evaluation of policy performance and feedback selection.} Seven-task OpenVLA-OFT comparisons on RoboTwin 2.0 show higher mean success than Full GRPO, complemented by localization and feedback-use measurements. Two-task ablations favor event-guided truncation over uniform penalty scaling and random prefixes.
\end{itemize}

%% file: figures/overview.tex
\begin{figure}[t]
 \centering
 \includegraphics[width=\linewidth]{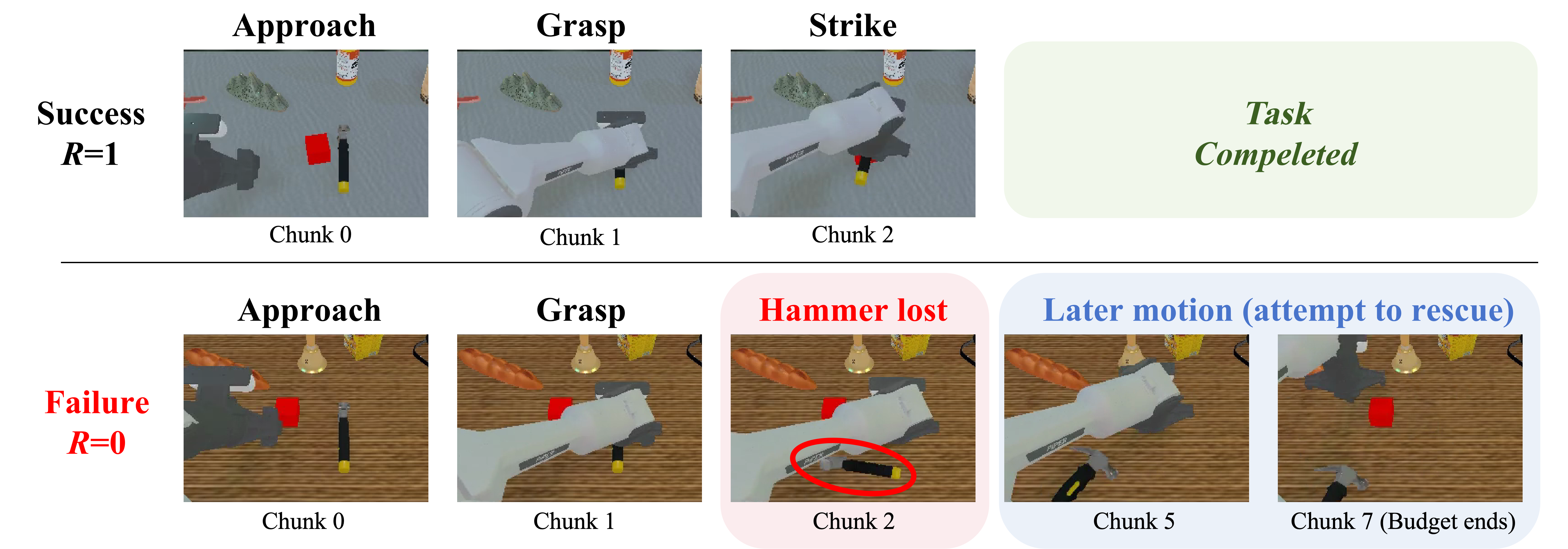}
 \caption{\textbf{Errors and recovery attempts within a failed rollout.} Selected RoboTwin hammer frames from two scenes, a success (top) and a failure (bottom). The red outline marks the hammer after it is lost at chunk 2, and the later frames show recovery attempts. In mixed-outcome groups, full GRPO assigns the error and subsequent recovery attempts the same negative advantage.}
 \label{fig:overview}
 \label{fig:event_example}
\end{figure}

%% file: sections/problem_setup.tex
\section{A Martingale Analysis of Failure Feedback}
\label{sec:setup}

We analyze feedback truncation through the gradient contribution of the discarded suffix. Under the conditions below, removing a contribution with zero conditional expectation given the retained information preserves the full-feedback estimator's expectation without increasing covariance. Two martingales make this statement precise for GRPO, and a projection residual then covers boundaries selected from complete rollouts.

\subsection{GRPO with a feedback mask}

We express feedback truncation as a mask on GRPO's trajectory-level advantage. Trajectory $i$ contains $K_i$ action chunks and terminal reward $R_i\in\{0,1\}$. History $h_{i,k}$ precedes chunk $a_{i,k}$, indexed from zero. A policy conditions on observations, task instructions, and available proprioception. For rollout group $g(i)$, the reference advantage is
\begin{equation}
 A_i=\frac{R_i-\overline R_{g(i)}}{s_{g(i)}+\epsilon_{\mathrm{num}}},\qquad
 \overline R_g=\frac{1}{|g|}\sum_{j\in g}R_j.
 \label{eq:group_advantage}
\end{equation}
Here $s_g$ is the group's reward standard deviation and $\epsilon_{\mathrm{num}}>0$ ensures numerical stability. The mask modifies the temporal support of $A_i$ as follows.
\begin{equation}
 \begin{aligned}
 \widetilde A_{i,k}&=w_{i,k}A_i,\qquad B_i\in\{-1,0,\ldots,K_i-1\},\\
 w_{i,k}&=\begin{cases}1,&R_i=1,\\ \ind\{k\le B_i\},&R_i=0.\end{cases}
 \end{aligned}
 \label{eq:prefix_mask}
\end{equation}
In the analysis, $B_i=-1$ represents an empty failure prefix. Let $\mathcal V_i$ be the valid action tokens and $N_i=|\mathcal V_i|$. Masking preserves the original $N_i$, since renormalizing by the retained tokens would scale up every surviving term. Section~\ref{sec:method} gives the clipped policy objective.

\subsection{Remaining conditional credit}

The ideal analysis concerns on-policy score gradients at a fixed differentiable collection policy $\pi_\theta$, before PPO clipping or optimizer transformations. Use a finite horizon $H$, padding after termination with zero-score dummy actions. The objective is undiscounted terminal success. The environment and initial distribution do not depend on $\theta$. Let $\F_k$ contain the complete pre-action history, including remaining budget, and let $\F_H$ contain the outcome. Assume common policy support, valid differentiation under expectation, and square-integrable scores
\begin{equation}
 P_k=\E[R\mid\F_k],\qquad
 Z_k=\nabla_\theta\log\pi_\theta(a_k\mid h_k),\qquad
 S_t=\sum_{k=0}^{t-1}Z_k.
 \label{eq:martingale_defs}
\end{equation}
The score identity $\E[Z_k\mid\F_k]=0$ makes $S_t$ a martingale, while $P_k$ is the Doob martingale of $R$~\citep{durrett2019probability}. Write $\Delta P_k=P_{k+1}-P_k$. Its expectation given the history and the selected action is the local action advantage, distinct from the group coefficient $A_i$ (Appendix~\ref{app:innovations}). This advantage records how the action moved the success probability, so an action inside a failed rollout can still carry positive credit.

Pairing their increments gives $\nabla J=\sum_k\E[Z_k\Delta P_k]$ for $J(\theta)=\E[R]$ by score centering~\citep{williams1992reinforce,sutton1999policy} (Appendix~\ref{app:innovations}). Summing the same products from chunk $t$ onward and conditioning on the observed history gives the remaining conditional credit
\begin{equation}
 \chi_t=\E[R(S_H-S_t)\mid\F_t]
 =\E\!\left[\sum_{k=t}^{H-1}Z_k\Delta P_k\,\middle|\,\F_t\right].
 \label{eq:remaining_credit}
\end{equation}
Holding the observed prefix fixed, $\chi_t$ is the first-order sensitivity of continuation success to the policy parameters (Appendix~\ref{app:remaining_credit}). With random initial conditions, $\E[\chi_0]=\nabla J$, while $\chi_0$ itself is conditional on $\F_0$. Contributions can cancel, so zero credit does not require every later action advantage to vanish.

\subsection{Conditional projection of a complete GRPO group}

We first consider rules that decide from the observed prefix. Let $T\in\{0,\ldots,H\}$ be a stopping time, with $\{T\le k\}\in\F_k$. It selects the last penalized failure chunk $B=T-1$. Successes retain their full scores, giving $\widetilde S_T=(1-R)S_T+RS_H$. Let $\mathcal K_T$ record the stopped history and final outcome, together with the complete trajectory on success. This retained information generally exceeds $\F_T$.

\begin{proposition}[A conditional projection characterization]
\label{prop:projection}
Under the score assumptions above,
\begin{equation}
 \chi_T=0\quad\Longleftrightarrow\quad
 \widetilde S_T=\E[S_H\mid\mathcal K_T].
 \label{eq:projection_characterization}
\end{equation}
\end{proposition}
Optional sampling gives $\E[(1-R)(S_H-S_T)\mid\F_T]=-\chi_T$, from which both directions follow (Appendix~\ref{app:projection}). The condition is imposed history by history. Preserving the unconditional expectation alone permits cancellation across histories and does not require $\chi_T=0$ (Appendix~\ref{app:counterexample}).

\begin{theorem}[Stopping failure feedback in a complete GRPO group]
\label{thm:group_stop}
Conditionally on the initial scene $X$, assume $m\ge2$ independent rollouts, each with $X$ in its local initial information and a local stopping time $T_i$ satisfying $\chi^{(i)}_{T_i}=0$. Let $\mathcal K=\bigvee_i\mathcal K^{(i)}_{T_i}$, and assume positive token counts $N_i$ are $\mathcal K$-measurable. The square-integrable group estimators satisfy
\begin{equation}
 G_g=\frac1m\sum_i\frac{A_i}{N_i}S_H^{(i)},\qquad
 G_{g,T}=\frac1m\sum_i\frac{A_i}{N_i}\widetilde S_{T_i}^{(i)}
        =\E[G_g\mid\mathcal K].
 \label{eq:group_stopping}
\end{equation}
Their expectations agree. With $D_g=G_g-G_{g,T}$,
\begin{equation}
 \Cov(G_g)-\Cov(G_{g,T})=\E[D_gD_g^\top]\succeq0.
 \label{eq:stopping_result}
\end{equation}
The same projection and covariance identity hold conditional on any positive-probability acceptance event determined by $X$ and the group outcomes.
\end{theorem}

Since $\mathcal K$ contains all outcomes and the token counts $N_i$ are $\mathcal K$-measurable, the coefficients $A_i/N_i$ factor out of the conditional expectation. This accommodates the dependence among group advantages (Appendix~\ref{app:group_stop}).

The theorem preserves the \emph{population expectation} $\mu=\E[G_g]$. Group-relative weighting and token normalization can make $\mu$ differ from $\nabla J$. Figure~\ref{fig:theory_geometry}a illustrates the projection by grouping full-gradient realizations by retained information, with hollow markers showing conditional expectations.

The finite-batch estimate is the sample average $\overline G=n^{-1}\sum_jG_{g_j}$. On the same $n$ independent groups, ideal masking gives $\overline G_T=\overline G-\overline D$, where $\overline D=n^{-1}\sum_jD_{g_j}$. Since $\E[\overline D]=0$, the two estimates share expectation $\mu$ even when their realized values differ. Figure~\ref{fig:theory_geometry}b shows the contributions forming a batch estimate. Masking zeros the orange suffix terms while preserving the original GRPO coefficients and normalization. Large markers show the full and masked batch estimates, and the star marks $\mu$. Ideal masking reduces mean-squared error (MSE) relative to $\mu$ by $\E\|D_g\|^2/n$ on average over batches, while individual batch errors can increase or decrease.

\input{figures/theory_geometry}

\begin{corollary}[Failure certification as a special case]
\label{cor:failure_stop}
If $P_T=0$ on $\{T<H\}$, then $\chi_T=0$, $\widetilde S_T=S_T$, and $\mathcal K_T=\F_T$ up to null events. The group theorem therefore includes local failure-certifying stopping times.
\end{corollary}

For stopping times satisfying $\chi_T=0$, earlier stopping has no larger covariance for the same coefficients and denominators (Appendix~\ref{app:nested_stop}). Certification is sufficient but unnecessary, since an action-independent outcome can have $P_T>0$ and zero remaining credit (Appendix~\ref{app:counterexample}).

\paragraph{Selecting candidate boundaries.}
Failure certification motivates looking for errors after which recovery is unlikely within the remaining budget. Visible errors and recovery observations serve as cues for candidate boundaries, whose statistical effects are characterized below.

\subsection{From ideal stops to retrospective boundaries}

Later observations can reveal whether an apparent error is followed by recovery. A boundary selected from the complete rollout can therefore depend on events beyond the retained prefix. Such selection can change the estimator's expectation even when each fixed candidate time has zero remaining credit (Appendix~\ref{app:counterexample}). We quantify this effect through a conditional projection residual.

Fix the policy, boundary rule, and reference sampling distribution. Let $G$ be the full group estimator, $G_B$ its masked version, and $D_B=G-G_B$. Let $\mathcal K_B$ contain the scene, all outcomes, the selected boundaries, actual denominators, retained failure prefixes, and complete success records. Thus $G_B$ is measurable in $\mathcal K_B$. The boundaries themselves must be included because they can convey future information. Define
\begin{equation}
 \eta_B=\E[D_B\mid\mathcal K_B],\qquad
 \epsilon_B^2=\E\|\eta_B\|^2.
 \label{eq:projection_residual}
\end{equation}

\begin{proposition}[Projection residual of a retrospective rule]
\label{prop:retrospective_projection}
For square-integrable $G$ and $G_B$,
\begin{equation}
 G_B=\E[G\mid\mathcal K_B]-\eta_B.
 \label{eq:retrospective_projection}
\end{equation}
Consequently $\eta_B=0$ is equivalent to this projection property and implies equal expectations and covariance nonincrease. For nonzero residual, write $V_B=\tr\Cov(G_B)$, $V_D=\tr\Cov(D_B)$, and $\Delta_V=\tr\Cov(G)-V_B$. Then
\begin{equation}
 \|\E[G_B]-\E[G]\|^2\le\epsilon_B^2,\qquad
 \Delta_V\ge V_D-2\sqrt{V_B\epsilon_B^2}.
 \label{eq:retrospective_bounds}
\end{equation}
\end{proposition}

The proposition follows from conditional expectation and Cauchy--Schwarz (Appendix~\ref{app:retrospective_projection}) under a fixed unfiltered or accepted-group law. The ideal stops of Theorem~\ref{thm:group_stop} satisfy $\eta_B=0$, and the residual measures a rule's departure from their conditional projection property.

\paragraph{Statistical benefit and truncation cost.}
For a finite batch, estimation accuracy depends on the balance between variance reduction and a shift in expectation. Write $G_w=G-D$ for a masked estimator and its removed part, so that $G_B$ and $D_B$ above are one instance. Set $\mu=\E[G]$ and $d=\E[D]$. Here bias is measured relative to the reference expectation $\mu$.

With $\Delta_V=\tr\Cov(G)-\tr\Cov(G_w)$, the standard bias--variance decomposition gives, for averages of $n$ independent copies,
\begin{equation}
 \E\|\overline G_w-\mu\|^2-\E\|\overline G-\mu\|^2
 =\|d\|^2-\frac{\Delta_V}{n}.
 \label{eq:mse}
\end{equation}
For an ideal stop in Theorem~\ref{thm:group_stop}, $d=0$ and the MSE reduction is $\E\|D_g\|^2/n$, which is strict when the deleted component is nonzero. For a general boundary, strict MSE improvement is equivalent to $\Delta_V>n\|d\|^2$, the region above the diagonal in Figure~\ref{fig:theory_geometry}(c) (Proposition~\ref{prop:mse} in the appendix).

Equation~\ref{eq:mse} balances variance reduction against the squared expectation shift at a given batch size. Appendix~\ref{app:hindsight} compares complete-rollout selection with rules based on the observed prefix. The next section constructs boundaries from visible errors and recovery.

%% file: figures/theory_geometry.tex
\begin{figure}[!t]
 \centering
 \includegraphics[width=\linewidth]{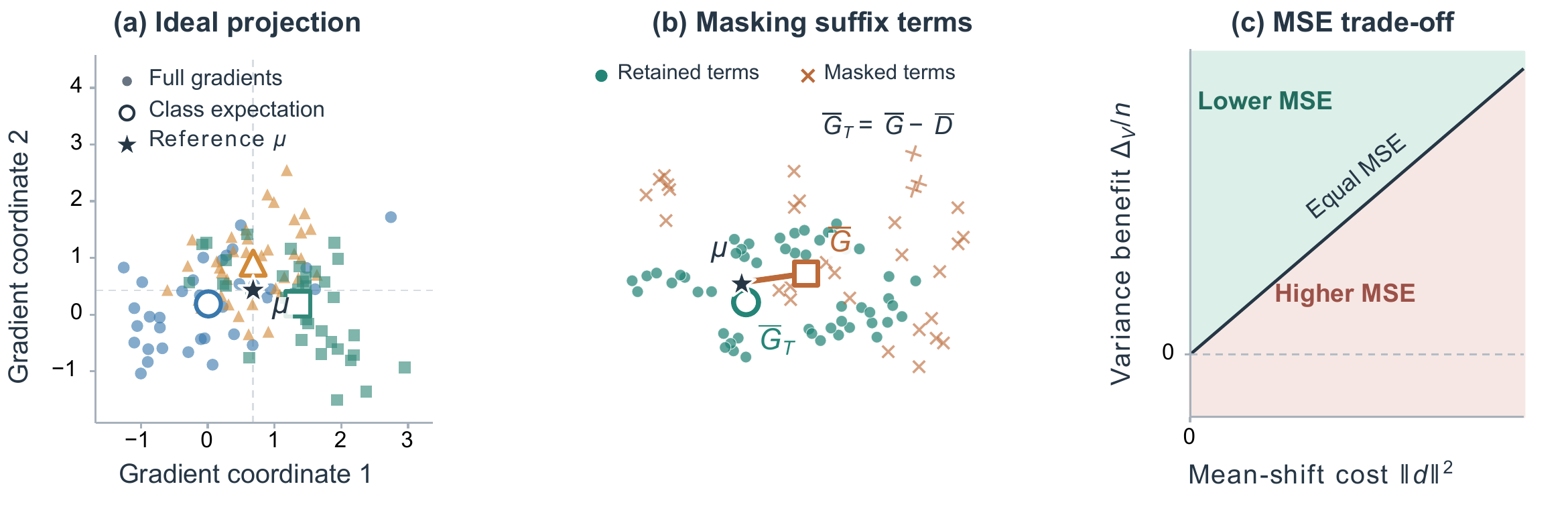}
 \caption{\textbf{Feedback truncation and gradient estimation.} Schematic views of (a) ideal conditional projection, (b) a finite-batch estimate before and after suffix masking, and (c) the trade-off between variance reduction and expectation shift.}
 \label{fig:theory_geometry}
\end{figure}

%% file: sections/method.tex
\par
\begin{samepage}
\section{Martingale GRPO}
\label{sec:method}

Section~\ref{sec:setup} identifies conditions for preserving the reference gradient expectation when removing a failure suffix. The failure-certification case (Corollary~\ref{cor:failure_stop}) motivates using visible errors and recovery as cues for candidate boundaries. \martgrpo\ retains feedback on localized errors and uses subsequent recovery observations to select the suffix to mask (Figure~\ref{fig:method}). A frozen VLM reviews each complete failed rollout and localizes the last visible error without observed recovery. A fixed mapping converts this interval into a failure-prefix mask for GRPO.

\end{samepage}

\input{figures/method}

\subsection{Learning to localize failure events}

\paragraph{Event supervision.}
The VLM input $x$ combines task instructions, three synchronized camera views, proprioceptive observations, and frame-to-chunk alignment. Structured targets record visible outcomes, error types, localization intervals, supporting frames, and observed recovery. The interval brackets the onset of the target error, since its exact chunk is uncertain.

The corpus contains 1,024 failed rollout videos with event annotations, supplemented by successful reference videos. Training, validation, and test partitions have disjoint task--scene identities. Annotation details and split counts are documented in Appendix~\ref{app:labels}.

\paragraph{VLM fine-tuning.}
A single LoRA adapter~\citep{hu2021lora} is trained with the Qwen3.8-27B vision-language backbone~\citep{qwen2026qwen38} frozen. For each structured response $y$, we average the negative log likelihood over response tokens, then over training examples:
\begin{equation}
 \mathcal L_{\mathrm{event}}(\phi)
 =-\frac{1}{|\mathcal D_{\mathrm{train}}|}
 \sum_{(x,y)\in\mathcal D_{\mathrm{train}}}
 \frac{1}{|y|}\sum_{v=1}^{|y|}
 \log p_\phi(y_v\mid x,y_{<v}).
 \label{eq:student}
\end{equation}
We select the checkpoint by validation task-macro interval intersection-over-union (IoU), breaking ties by lower response loss and then earlier update, and freeze it during RL. Appendix~\ref{app:vlm8_training} gives the training and evaluation protocol.

\subsection{Selecting the feedback boundary}

For a failed rollout, the boundary uses the following fields from the VLM's structured response:
\begin{equation}
 \widehat e_i=f_\phi(x_i)=(o_i,\mathrm{type}_i,[\ell_i,u_i],r_i,\mathcal I_i).
 \label{eq:events}
\end{equation}
Here $o_i$ is the predicted visible outcome, $r_i$ records subsequent recovery from the localized error, and $\mathcal I_i$ identifies supporting frames for inspection. Response validation checks the schema and requires any non-null interval to contain two integer indices satisfying $0\le\ell_i\le u_i<K_i$. Invalid responses and failed inference requests raise implementation errors (Appendix~\ref{app:event_rule}). For a valid response, truncation requires a predicted failure, a specific error type, a non-null interval, and an explicit report of no recovery. Other valid responses retain full feedback, which we call a semantic fallback:
\begin{equation}
 B_i=
 \begin{cases}
  u_i,&\text{a localized event passes the checks},\\
  K_i-1,&\text{a semantic fallback applies}.
 \end{cases}
 \label{eq:event_boundary}
\end{equation}
For a qualifying event, $[\ell_i,u_i]$ localizes the error onset, while feedback is retained on the full prefix $0,\ldots,B_i$ with $B_i=u_i$. This choice includes every candidate onset chunk. For example, $[1,2]$ in a four-chunk failed rollout retains chunks $0,1,2$, giving mask $(1,1,1,0)$ under Equation~\ref{eq:prefix_mask}.

\subsection{Policy optimization}

At each policy update, rollouts execute to their terminal condition or budget. Following Full GRPO, we retain groups containing both successes and failures and collect further groups to fill the batch. Terminal rewards determine group advantages. The frozen VLM and fixed boundary mapping determine masks for the new failed rollouts, while successes bypass event inference and retain all their loss terms.

Let $k(u)$ denote the chunk containing token $u$, and $\rho_{i,u}(\theta)$ the token probability ratio relative to the collection policy. For $M$ rollouts and clipping limits $1-\varepsilon_-,1+\varepsilon_+$, we optimize
\begin{equation}
 \mathcal J_w(\theta)=\frac{1}{M}\sum_{i=1}^{M}\frac{1}{N_i}
 \sum_{u\in\mathcal V_i}\min\!\left\{
 \rho_{i,u}\widetilde A_{i,k(u)},\,
 \operatorname{clip}(\rho_{i,u},1-\varepsilon_-,1+\varepsilon_+)
 \widetilde A_{i,k(u)}\right\}.
 \label{eq:ppo}
\end{equation}
Masks remain fixed and detached throughout each PPO update, and the original valid-token counts $N_i$ are retained. Deployment uses the actor alone. Appendix~\ref{app:feedback_normalization} defines the retained coefficient mass used to quantify this feedback.

\paragraph{Connection to the analysis.}
The VLM uses later observations to interpret errors and recovery, so $B$ is retrospective. Proposition~\ref{prop:retrospective_projection} characterizes these masks through a conditional projection residual. For these boundaries, the ideal projection guarantee requires $\eta_B=0$, an additional condition beyond visible-event localization. Policy comparisons in Section~\ref{sec:experiments} assess learning under the clipped PPO objective.

%% file: figures/method.tex
\begin{figure}[!ht]
 \centering
 \begin{tikzpicture}
  \node[anchor=south west,inner sep=0] (workflow) at (0,0)
   {\includegraphics[width=\linewidth]{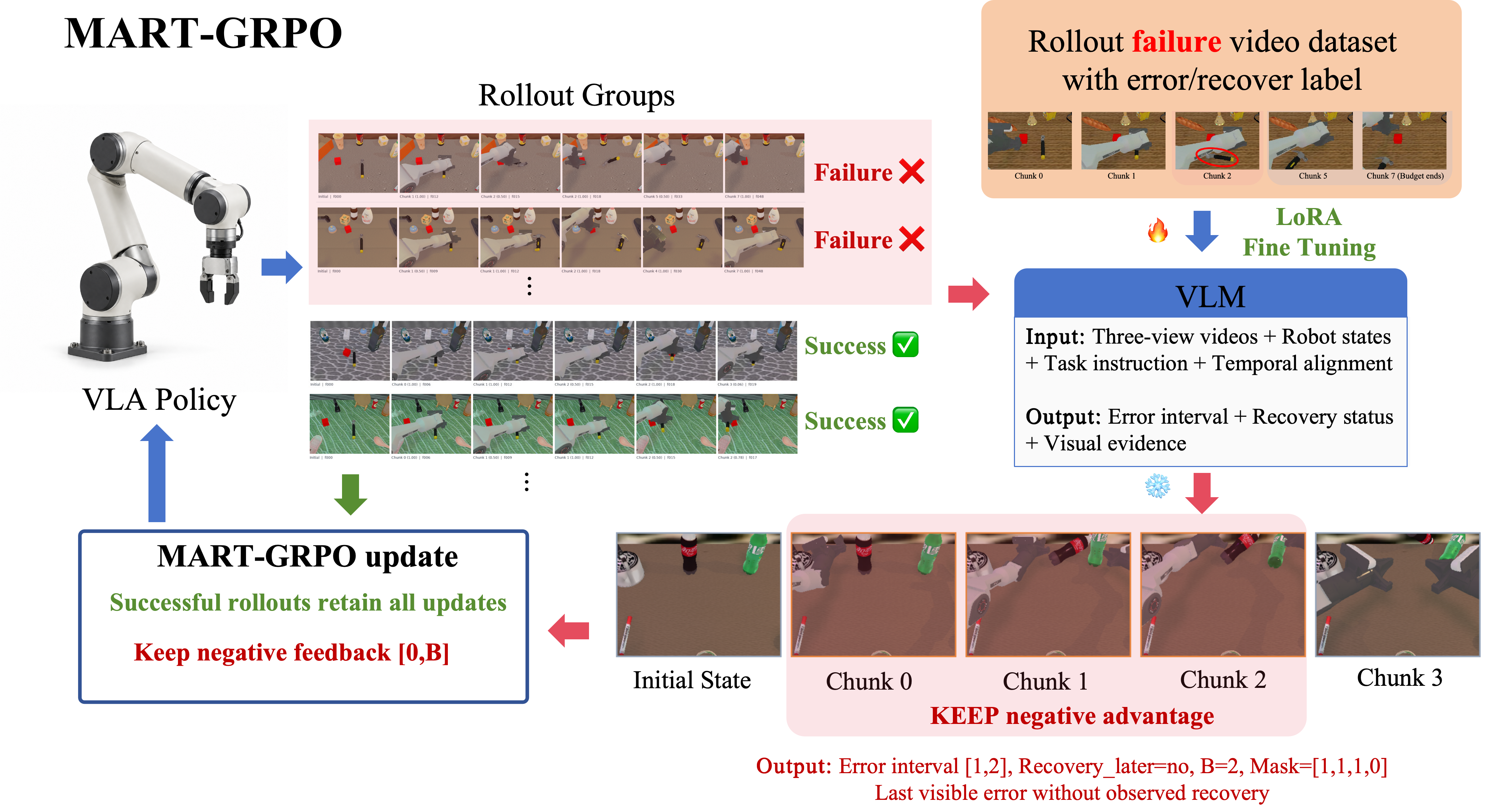}};
  \begin{scope}[x={(workflow.south east)},y={(workflow.north west)}]
   \fill[white] (0.067,0.241) rectangle (0.348,0.283);
   \node[inner sep=0,text={rgb,255:red,86;green,143;blue,46},
     font=\bfseries\fontsize{7}{8.5}\selectfont]
     at (0.205,0.258) {Successful-rollout loss terms retained};
   \fill[fill={rgb,255:red,248;green,203;blue,169}]
     (0.680,0.878) rectangle (0.972,0.981);
   \node[inner sep=0,align=center,font=\fontsize{9}{11}\selectfont]
     at (0.826,0.928) {Annotated rollout videos\\with error and recovery labels};
  \end{scope}
 \end{tikzpicture}
 \caption{\textbf{MART-GRPO: from error localization to feedback truncation.} A VLM is fine-tuned offline on event annotations, then frozen during RL. It localizes the last visible error without observed recovery in complete failed rollouts. The interval's right endpoint sets $B$: retain negative feedback through $B$ and mask the suffix. Success terms and original token denominators are preserved.}
 \label{fig:method}
\end{figure}

%% file: sections/experiments.tex
\section{Evaluation}
\label{sec:experiments}

We evaluate VLM event localization, VLA policy performance, and two ablations on the amount and temporal placement of failure feedback.

\subsection{VLM event localization}
\label{sec:prefix_evaluation}

We evaluate 220 unrecovered test failures with annotated intervals from the seven policy tasks, using the task--scene-disjoint split in Section~\ref{sec:method}. Interval intersection-over-union (IoU) compares discrete chunk sets. Endpoint accuracy measures agreement within a given chunk distance of the annotated right endpoint.

\begin{samepage}
Both metrics use predictions before semantic fallback, score missing intervals as zero, and average within tasks and then across tasks. Validation IoU selects update 304 (epoch 8). These scores measure exploratory agreement with the reference annotations (Appendix~\ref{app:vlm8_training}).
\end{samepage}

Figure~\ref{fig:vlm_localization} compares the base model with epochs 3, 6, and 8. Task-macro IoU rises from $28.5\%$ to $50.9\%$, $67.8\%$, and $72.1\%$, respectively. The selected checkpoint achieves $68.1\%$ exact endpoint agreement and $91.8\%$ agreement within one chunk. Per-task values appear in Table~\ref{tab:vlm8_localization}.

\input{figures/vlm_localization}

\subsection{VLA policy performance}
\label{sec:online_pilot}

\paragraph{Setup.}
We compare SFT, Full GRPO, and \martgrpo\ on seven RoboTwin 2.0 tasks~\citep{chen2025robotwin}, using the action-token implementation of OpenVLA-OFT from SimpleVLA-RL~\citep{kim2025oft,li2025simplevla}. Both RL methods start from the same task-specific SFT checkpoint. Each update uses 64 accepted groups of eight rollouts, mixed-outcome filtering with adaptive refill, global PPO minibatch 128, and the original token normalization. The actor receives a head-camera view and proprioception, while the frozen VLM reviews three views. Experiments run on servers equipped with eight H100 GPUs. Greedy evaluation uses 128 training-seed (Seen) and 128 held-out-seed episodes under the same task configuration every five updates. The primary endpoint is update 40. Full settings and scheduled evaluations appear in Appendices~\ref{app:grpo8} and~\ref{app:epoch8_policy}.

\input{tables/policy_results}

\paragraph{Results.}
At update 40, Table~\ref{tab:eight_task_policy} shows mean success rising from $65.01\%$ to $68.64\%$ ($+3.63$ points), with gains of $1.23$ points on seen seeds and $6.03$ on held-out seeds. Six tasks improve, led by Handover ($+8.98$) and Hammer ($+7.03$). Mean per-task best scores rise from $66.02\%$ to $69.87\%$. Container on Plate starts at $87.11\%$ SFT success, leaving limited headroom, and both methods peak at $94.92\%$. On Container on Plate, MART's pooled endpoint success is $1.95$ points below Full GRPO, with the difference confined to held-out seeds.

\subsection{Feedback ablations}
\label{sec:policy_ablations}

We conduct ablations on two tasks, Hammer and Handover. Two ablations use statistics frozen from all 40 updates of the corresponding MART run. Historical MART retains $49.2\%$ and $32.7\%$ of negative coefficient mass (Equation~\ref{eq:mass}), respectively, measuring feedback weight rather than gradient norm. \emph{Uniform scaling} uses these ratios ($0.4921$ and $0.3269$) to scale every failed-trajectory advantage. \emph{Random prefix} samples boundaries independently from the historical MART boundary distribution. Both ablations preserve success weights and original token denominators, use no online VLM, and follow the same 40-update budget and evaluation schedule. Protocol details appear in Appendix~\ref{app:two_task_ablation}.

Table~\ref{tab:policy_ablations} reports success at updates 10, 20, 30, and 40, and Figure~\ref{fig:hammer_handover_learning} compares final results alongside the SFT references. Uniform scaling reaches $77.73\%$ on Hammer and $40.62\%$ on Handover, improving over Full GRPO. MART remains higher by $3.52$ and $3.12$ points, respectively. Random prefix reaches $71.88\%$ and $41.41\%$, trailing MART by $9.38$ and $2.34$ points. VLM-guided truncation thus achieves higher success than simply reducing failure penalties on both tasks. Its advantage over random prefixes further supports selecting boundaries from visible events.

\input{tables/policy_ablations}
\input{figures/policy_endpoints}

%% file: figures/vlm_localization.tex
\begin{figure}[!t]
 \centering
 \includegraphics[width=\linewidth,trim=0bp 12bp 0bp 0bp,clip]{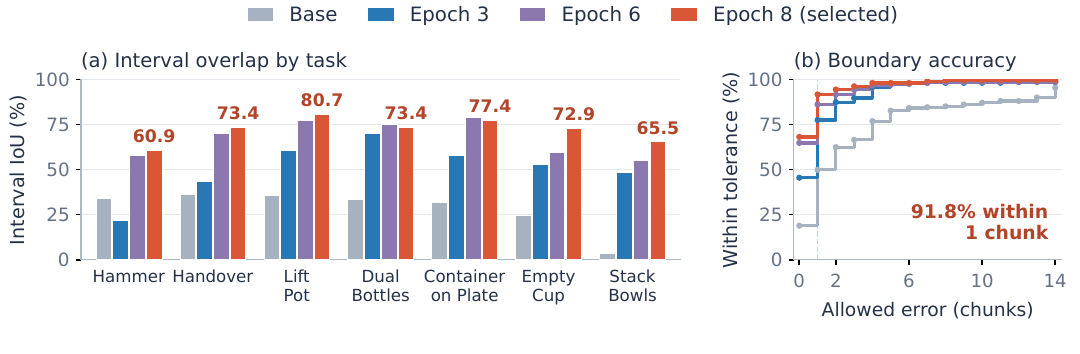}
 \caption{\textbf{VLM event localization.} Per-task interval IoU (a) and cumulative endpoint accuracy (b) on test failures. Epoch 8 is validation-selected.}
 \label{fig:vlm_localization}
\end{figure}

%% file: tables/policy_results.tex
\begin{table}[!t]
 \setlength{\abovecaptionskip}{0pt}
 \caption{\textbf{Seven-task policy performance.} Success (\%) on 128 training-seed (Seen) and 128 held-out-seed (Held-out) episodes. All pools both splits. RL scores are at update 40 except Best, the maximum All over updates $5,10,\ldots,40$. SFT is the initial score. $\Delta_{40}$ and $\Delta_{\mathrm{Best}}$ give MART-GRPO gains over Full GRPO (percentage points). Means and differences are computed before rounding.}
 \label{tab:eight_task_policy}
 \centering\small
 \setlength{\tabcolsep}{1.5pt}
 \begin{tabular*}{\linewidth}{@{\extracolsep{\fill}}lrrrrrrrrrrr@{}}
 \toprule
 & SFT & \multicolumn{4}{c}{Full GRPO~\citep{li2025simplevla}} & \multicolumn{4}{c}{MART-GRPO} & & \\
 \cmidrule(lr){3-6}\cmidrule(lr){7-10}
 Task & All & Seen & Held-out & All & Best & Seen & Held-out & All & Best & $\Delta_{40}$ & $\Delta_{\mathrm{Best}}$\\
 \midrule
Hammer & 30.47 & 78.12 & 70.31 & 74.22 & 74.61 & 82.81 & 79.69 & \textbf{81.25} & \textbf{81.25} & +7.03 & +6.64\\
Handover & 21.09 & 38.28 & 31.25 & 34.77 & 35.55 & 40.62 & 46.88 & \textbf{43.75} & \textbf{44.53} & +8.98 & +8.98\\
Lift pot & 5.47 & 60.16 & 56.25 & 58.20 & 58.20 & 57.03 & 60.16 & \textbf{58.59} & \textbf{58.59} & +0.39 & +0.39\\
Dual bottles & 20.70 & 50.00 & 39.84 & 44.92 & 44.92 & 53.91 & 43.75 & \textbf{48.83} & \textbf{51.17} & +3.91 & +6.25\\
Container on Plate & 87.11 & 92.97 & 94.53 & \textbf{93.75} & \textbf{94.92} & 92.97 & 90.62 & 91.80 & \textbf{94.92} & -1.95 & +0.00\\
Empty cup & 80.86 & 91.41 & 86.72 & 89.06 & 90.62 & 92.19 & 90.62 & \textbf{91.41} & \textbf{91.41} & +2.34 & +0.78\\
Stack bowls & 39.06 & 63.28 & 57.03 & 60.16 & 63.28 & 63.28 & 66.41 & \textbf{64.84} & \textbf{67.19} & +4.69 & +3.91\\
\midrule
Mean (7 tasks) & 40.68 & 67.75 & 62.28 & 65.01 & 66.02 & 68.97 & 68.30 & 68.64 & 69.87 & +3.63 & +3.85\\
 \bottomrule
 \end{tabular*}
\end{table}

%% file: tables/policy_ablations.tex
\begin{table}[!ht]
 \setlength{\abovecaptionskip}{0pt}
 \caption{\textbf{Feedback ablations.} Success (\%) pooled across both seed splits at updates 10, 20, 30, and 40. Column bests are bold. Reported differences are computed before rounding.}
 \label{tab:policy_ablations}
 \centering\small
 \setlength{\tabcolsep}{2pt}
 \begin{tabular*}{\linewidth}{@{\extracolsep{\fill}}llrrrrrrrr@{}}
 \toprule
 Method & Failure feedback & \multicolumn{4}{c}{Hammer} & \multicolumn{4}{c}{Handover}\\
 \cmidrule(lr){3-6}\cmidrule(l){7-10}
 & & 10 & 20 & 30 & 40 & 10 & 20 & 30 & 40\\
 \midrule
 Full GRPO & Full trajectory & 48.05 & 53.91 & 65.23 & 74.22 & \textbf{32.03} & 31.25 & 35.55 & 34.77\\
 Uniform scaling & Uniformly scaled & \textbf{51.56} & 63.28 & 65.23 & 77.73 & 29.69 & 33.98 & 37.89 & 40.62\\
 Random prefix & Random prefix & 51.17 & 67.19 & 75.00 & 71.88 & 25.78 & 32.03 & 38.67 & 41.41\\
 MART-GRPO & Event-guided prefix & 49.61 & \textbf{69.14} & \textbf{78.12} & \textbf{81.25} & 25.39 & \textbf{35.16} & \textbf{42.58} & \textbf{43.75}\\
 \bottomrule
 \end{tabular*}
\end{table}

%% file: figures/policy_endpoints.tex
\begin{figure}[!ht]
 \centering
 \includegraphics[width=0.90\linewidth]{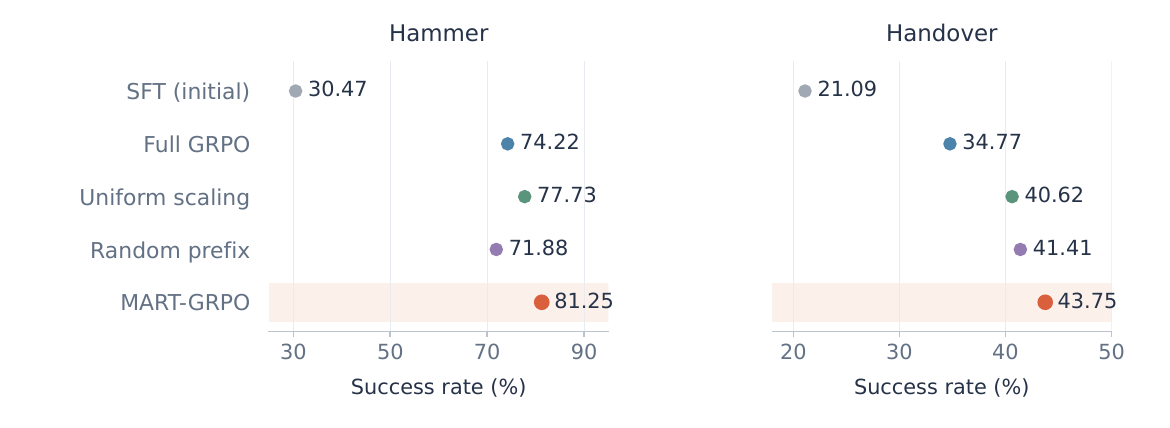}
 \caption{\textbf{Final performance on Hammer and Handover.} Success after 40 updates, with SFT references and task-specific horizontal scales.}
 \label{fig:hammer_handover_learning}
\end{figure}

%% file: sections/conclusion.tex
\section{Conclusion and Limitations}
\label{sec:conclusion}

We study where failure feedback should end in VLA reinforcement learning. Two martingales characterize remaining conditional credit and identify ideal stops that preserve the reference group-gradient expectation without increasing covariance. Guided by this analysis, \martgrpo\ selects feedback boundaries from visible errors and recovery in complete rollout videos. Retrospective residuals characterize departures of these boundaries from the ideal projection. Policy evaluations show success-rate improvements over Full GRPO across multiple tasks. On both ablation tasks, VLM-guided truncation achieves higher endpoint success than uniform penalty scaling and random prefix selection. A key open question is how reliably visible events track remaining credit as the policy evolves. Future work will examine this relationship across policy stages and develop prefix-based event predictors.

%% file: sections/statements.tex
\subsection*{AI use statement}

Generative AI tools assisted method implementation, experiment planning, and result analysis. They also supported language polishing, figure preparation, and checks of mathematical proofs. The authors reviewed the AI-assisted work and take full responsibility for the final manuscript.

\subsection*{Reproducibility statement}

Appendix~\ref{app:proofs} supplies the stopping and retrospective-projection proofs, together with the hindsight decomposition and the approximate-stopping bound. Finite-model checks supplement the proofs. Appendix~\ref{app:grpo8} gives the policy training and evaluation settings, and Appendix~\ref{app:labels} documents the boundary mapping and the annotation corpus. Appendix~\ref{app:epoch8_policy} provides scheduled policy evaluations and feedback statistics, and Appendix~\ref{app:two_task_ablation} details the ablation protocol.

%% file: appendix/related_work.tex
\section{Related Work}
\label{sec:related}
\label{app:related}

\paragraph{VLA models and reinforcement learning.}
Vision-language pretraining supports robot control in RT-2 and OpenVLA~\citep{zitkovich2023rt2,kim2025openvla}, while $\pi_0$ uses flow-based continuous action generation~\citep{black2025pi0}. OpenVLA-OFT studies efficient supervised adaptation~\citep{kim2025oft}. VLA-RL and SimpleVLA-RL develop RL pipelines for pretrained policies~\citep{lu2025vlarl,li2025simplevla}, with further work examining RL generalization~\citep{liu2025rlvla}. We select which action chunks receive terminal failure feedback while retaining the policy architecture and environment reward.

\paragraph{Temporal credit assignment.}
RUDDER redistributes delayed rewards through return prediction~\citep{arjona2019rudder}. Hindsight and counterfactual credit assignment use future information to construct action-credit estimates or baselines~\citep{harutyunyan2019hca,mesnard2020counterfactual}. Temporal GRPO reconstructs stage-relative advantages using task-stage detection and rollout alignment, while GiGPO estimates local credit at repeated anchor states~\citep{zhou2026temporal,feng2025gigpo}. These methods build a new credit signal. We keep the existing trajectory-level advantage and change only the actions it reaches.

\paragraph{Truncating negative feedback.}
A second group changes which update terms survive. PCM samples chunks using phase-level success--failure divergence to allocate gradient computation~\citep{bagaria2026pcm}. $\mu$-GRPO uses importance-ratio triggers to veto suffix updates, motivated by off-policy prefix mismatch~\citep{tian2026mugrpo}. ESPO uses a regret surrogate to terminate failing rollouts early and assign terminal failure~\citep{li2026espo}, whereas PS-PPO samples update prefixes with importance weighting to preserve the reference gradient expectation and study the variance--computation tradeoff~\citep{hwang2026psppo}. \martgrpo\ masks failed-rollout suffixes after complete execution, selecting boundaries from visible errors and recovery. Existing group advantages and loss normalization are retained without additional truncation correction weights.

\paragraph{Vision-language feedback and corrective learning.}
Vision-language feedback has been used to construct rewards from image--text similarity~\citep{rocamonde2024vlmreward}, video or textual demonstrations~\citep{sontakke2023roboclip}, and process-level predictions~\citep{lu2025vlarl}. REFLECT summarizes robot experiences for failure explanation and correction~\citep{liu2023reflect}. DLS uses privileged simulator states for progress-localized signals and directional shaping in flow-based VLA adaptation~\citep{chen2026dls}. RedFlow estimates chunk quality and derives corrective action targets from successful matches in offline data~\citep{yan2026redflow}. In \martgrpo, the fine-tuned VLM identifies the last visible error without observed recovery in complete videos. Its interval selects where an existing negative advantage ends, while the scalar task reward remains unchanged.

\paragraph{Gradient variance and normalization.}
REINFORCE and the policy-gradient theorem give the score-function foundation~\citep{williams1992reinforce,sutton1999policy}. Baselines and value-based advantage estimates address gradient variance~\citep{greensmith2004variance,schulman2016gae}, while GRPO and RLOO use groups of sampled outcomes~\citep{shao2024deepseekmath,ahmadian2024rloo}. Dr.\ GRPO analyzes normalization bias, and DAPO studies token-level aggregation and dynamic sampling~\citep{liu2025drgrpo,yu2025dapo}. These loss and normalization choices affect the reference update itself. We characterize failure-only truncation of complete GRPO groups through optional sampling and conditional-expectation projection~\citep{durrett2019probability}. The analysis covers dependent advantages, token denominators, and outcome filtering. Under the stopping and measurability conditions, the projection preserves the reference expectation without increasing covariance. This guarantee concerns the reference estimator before PPO clipping~\citep{schulman2017ppo}. Conditional projection residuals characterize expectation and variance changes under retrospective boundary selection.

%% file: appendix/appendix.tex
\section{Proofs and the Scope of the Martingale Argument}
\label{app:proofs}

\subsection{Notation and regularity}

We use the fixed-policy setting of Section~\ref{sec:setup}. A trajectory has finite horizon $H$ and binary terminal outcome $R$. Its distribution factors into an initial distribution, policy probabilities, and environment transitions, with only the policy factors depending on $\theta$. Differentiation is assumed to commute with expectation. For continuous actions, the policy has a differentiable density with common support. Scores are square-integrable, and terminal padding uses parameter-independent dummy actions with zero score.

Let $h_k$ denote the complete pre-action history used in the analysis, so that $\F_k=\sigma(h_k)$, and let $\F_H$ include $R$. The notation $\pi_\theta(a_k\mid h_k)$ denotes the action law conditional on this history, even when the policy uses only its observable component. Additional environment information may be included in $h_k$ provided that the conditional action law remains $\pi_\theta$. Thus $P_k=\E[R\mid\F_k]$ conditions on the analytical history, while the policy input specifies which part of that history is available to the agent. Under these assumptions, $Z_k$ is $\F_{k+1}$-measurable and satisfies $\E[Z_k\mid\F_k]=0$.

\subsection{Standard martingale identities}
\label{app:innovations}

\begin{proposition}[Outcome innovations and local credit]
\label{prop:innovations}
Under the fixed-policy score assumptions in Section~\ref{sec:setup}, $(P_k)_{k=0}^{H}$ is a bounded martingale. For $J(\theta)=\E[R]$ and every stopping time $T\le H$,
\begin{equation}
 \begin{aligned}
 \nabla J(\theta)&=\sum_{k=0}^{H-1}\E[Z_k\Delta P_k],\\
 \E\!\left[\sum_{k=T}^{H-1}(\Delta P_k)^2\,\middle|\,\F_T\right]
 &=P_T(1-P_T).
 \end{aligned}
 \label{eq:innovation_identity}
\end{equation}
\end{proposition}

\begin{proof}
Since $0\le R\le1$, each $P_k$ is integrable and lies in $[0,1]$. The tower property gives
\begin{equation}
 \E[P_{k+1}\mid\F_k]=\E[\E[R\mid\F_{k+1}]\mid\F_k]=P_k.
\end{equation}
Moreover $P_H=R$, because the terminal outcome is observed at $H$. The likelihood-ratio identity yields
\begin{equation}
 \nabla J(\theta)=\E\!\left[R\sum_{k=0}^{H-1}Z_k\right].
\end{equation}
For each $k$, measurability of $Z_k$ with respect to $\F_{k+1}$ and the score identity imply
\begin{align}
 \E[RZ_k]&=\E[P_{k+1}Z_k],\\
 \E[P_kZ_k]&=\E[P_k\E[Z_k\mid\F_k]]=0.
\end{align}
Subtracting gives the first identity by a conditional-expectation substitution at the fixed policy parameter.

For a deterministic $t$, increments $\Delta P_k$ after $t$ are mutually orthogonal conditional on $\F_t$. For $t\le j<k$, the tower property gives
\begin{equation}
 \E[\Delta P_j\Delta P_k\mid\F_t]
 =\E[\Delta P_j\E[\Delta P_k\mid\F_k]\mid\F_t]=0.
\end{equation}
The same equality at a bounded stopping time follows by partitioning over $\{T=t\}$. Expanding $R-P_T=\sum_{k=T}^{H-1}\Delta P_k$ therefore gives
\begin{equation}
 \E\!\left[\sum_{k=T}^{H-1}(\Delta P_k)^2\,\middle|\,\F_T\right]
 =\E[(R-P_T)^2\mid\F_T]=P_T-P_T^2,
\end{equation}
where $R^2=R$ is used in the last step. On any event where $P_T=0$, nonnegativity of $R$ implies $R=0$ almost surely conditional on $\F_T$.
\end{proof}

\paragraph{Action-conditioned credit.}
To interpret $\Delta P_k$, define $Q^\pi(h,a)=\E[R\mid h,a]$ for the undiscounted terminal-success objective. Since $\F_k=\sigma(h_k)$ and the action $a_k$ is observed by $k+1$, the tower property gives
\begin{equation}
 \E[P_{k+1}\mid\F_k,a_k]=\E[R\mid\F_k,a_k]=Q^\pi(h_k,a_k).
\end{equation}
Subtracting the current success probability $P_k$ yields the local action advantage
\begin{equation}
 \mathcal A^\pi(h_k,a_k)=Q^\pi(h_k,a_k)-P_k
 =\E[\Delta P_k\mid\F_k,a_k].
 \label{eq:action_advantage}
\end{equation}
Conditioning on $\F_k$ gives $\E[\mathcal A^\pi(h_k,a_k)\mid\F_k]=0$. Because $Z_k$ is a function of $(h_k,a_k)$, its expected credit is $\E[Z_k\mathcal A^\pi(h_k,a_k)]=\E[Z_k\Delta P_k]$. This connects each term in the gradient identity to the action's expected effect on terminal success.

\paragraph{Policy-relative success probability.}
The process $P_k$ is defined by the collection policy, the analytical history, and the remaining budget. Accordingly, $P_T=0$ certifies failure under that continuation law. Its first-order interpretation uses the common policy support assumed above.

\subsection{Remaining conditional credit}
\label{app:remaining_credit}

The same score substitution can be conditioned on an earlier history. For deterministic $t\le k$, score centering and the measurability of $Z_k$ at $k+1$ give
\begin{equation}
 \begin{aligned}
 \E[RZ_k\mid\F_t]&=\E[P_{k+1}Z_k\mid\F_t],\\
 \E[P_kZ_k\mid\F_t]&=0.
 \end{aligned}
\end{equation}
Subtracting the two equalities and summing over $k\ge t$ proves Equation~\ref{eq:remaining_credit}. Partitioning over $\{T=t\}$ extends this identity to a bounded stopping time. Optional sampling centers the remaining score sum, so
\begin{equation}
 \begin{aligned}
 \E[S_H-S_T\mid\F_T]&=0,\\
 \chi_T&=\E[(R-P_T)(S_H-S_T)\mid\F_T].
 \end{aligned}
\end{equation}
Thus $\chi_T$ is the conditional covariance between terminal success and the remaining score sum. Since it sums vector contributions, later action credits may cancel.

\paragraph{Continuation sensitivity and initial conditioning.}
For a fixed complete prefix $h_t$, let $V_\vartheta(h_t)$ be the success probability when future actions are sampled from $\pi_\vartheta$. The recorded prefix and conditional environment law are held fixed. Under the differentiation assumptions above, the suffix likelihood-ratio identity gives
\begin{equation}
 \begin{aligned}
 \left.\nabla_\vartheta V_\vartheta(h_t)\right|_{\vartheta=\theta}
 &=\E_\theta\!\left[R\sum_{k=t}^{H-1}Z_k\,\middle|\,h_t\right]\\
 &=\chi_t(h_t).
 \end{aligned}
 \label{eq:continuation_sensitivity}
\end{equation}
This derivative measures local sensitivity at the current policy parameter, so $\chi_t=0$ can arise through cancellation or stationarity. At the initial time, $\F_0=\sigma(X)$ makes $\chi_0$ the gradient conditional on scene $X$. Averaging over the parameter-independent initial distribution then gives the population gradient $\E[\chi_0]=\nabla J$.

\subsection{Proof of Proposition~\ref{prop:projection}}
\label{app:projection}

Fix a stopping time $T$ for the local filtration. To retain the stopped history on failure and the complete history on success, define
\begin{equation}
 \mathcal K_T=\sigma\!\left(\F_T,\,R,\,
       \{R\ind_A:A\in\F_H\}\right).
 \label{eq:retained_information}
\end{equation}
On failure, the last collection is identically zero, leaving the stopped history and failure label. On success, it reveals the complete trajectory. Therefore $\widetilde S_T=(1-R)S_T+RS_H$ is $\mathcal K_T$-measurable. This information set records what is retained after execution, while the stopping decision itself is determined by the local prefix.

\begin{proof}
Bounded optional sampling~\citep[Section~4.8]{durrett2019probability} yields
\begin{equation}
 \E[(1-R)(S_H-S_T)\mid\F_T]=-\chi_T.
 \label{eq:conditional_removed_score}
\end{equation}
Suppose $\chi_T=0$. At histories with $P_T<1$,
\begin{equation}
 \E[S_H-S_T\mid\F_T,R=0]=\frac{-\chi_T}{1-P_T}=0.
\end{equation}
Histories with $P_T=1$ have no failure branch. Hence the removed score has zero conditional mean given $\mathcal K_T$ on failure, and is identically zero on success. Since $S_H-\widetilde S_T=(1-R)(S_H-S_T)$, this proves
\begin{equation}
 \E[S_H\mid\mathcal K_T]=\widetilde S_T.
\end{equation}
Conversely, the projection implies $\E[S_H-\widetilde S_T\mid\mathcal K_T]=0$. Conditioning on $\F_T\subseteq\mathcal K_T$ and using Equation~\ref{eq:conditional_removed_score} gives $\chi_T=0$.
\end{proof}

\paragraph{Weighted estimators and scope of necessity.}
For a bounded $\mathcal K_T$-measurable coefficient $C$, the condition gives
\begin{equation}
 \begin{aligned}
 G&=CS_H,\qquad G_T=C\widetilde S_T,\\
 G_T&=\E[G\mid\mathcal K_T].
 \end{aligned}
\end{equation}
Square integrability then gives expectation preservation and covariance reduction. Necessity in Proposition~\ref{prop:projection} concerns the unweighted score, since a zero coefficient can hide nonzero conditional credit in a weighted estimator. Unconditional expectation preservation is weaker because conditional shifts can cancel across histories.

\subsection{Proof of Theorem~\ref{thm:group_stop}}
\label{app:group_stop}

Let $X$ denote the shared initial scene. The complete trajectory sigma-fields are conditionally independent given $X$. Each local filtration contains $X$ and that trajectory's history, without revealing sibling futures. Each $T_i\le H$ is a stopping time for its local filtration, with $\chi^{(i)}_{T_i}=0$. The $N_i$ are positive integer token counts known in $\mathcal K=\bigvee_i\mathcal K^{(i)}_{T_i}$.

\begin{proof}
Every $R_i$ is retained in $\mathcal K$, so each bounded GRPO coefficient $A_i$ in Equation~\ref{eq:group_advantage} is measurable there. This shared information accommodates the dependence among group advantages.

Conditional independence of the local trajectories, followed by Proposition~\ref{prop:projection}, gives
\begin{equation}
 \begin{aligned}
 \E[S_H^{(i)}-\widetilde S_{T_i}^{(i)}\mid\mathcal K]
 &=\E[S_H^{(i)}-\widetilde S_{T_i}^{(i)}\mid\mathcal K^{(i)}_{T_i}]\\
 &=0.
 \end{aligned}
 \label{eq:joint_score_centering}
\end{equation}
The first equality uses the fact that each sibling's retained record is a function of its own trajectory and $X$. Taking the known coefficients $A_i/N_i$ outside the conditional expectation and summing proves $G_{g,T}=\E[G_g\mid\mathcal K]$. With $D_g=G_g-G_{g,T}$ and $\mu_g=\E[G_g]$, conditional centering implies
\begin{equation}
 \E[(G_{g,T}-\mu_g)D_g^\top]
 =\E[D_g(G_{g,T}-\mu_g)^\top]=0.
\end{equation}
Expanding the covariance gives
\begin{equation}
 \begin{aligned}
 \E[G_{g,T}]&=\E[G_g],\\
 \Cov(G_g)-\Cov(G_{g,T})&=\E[D_gD_g^\top]\succeq0.
 \end{aligned}
\end{equation}

Let $\mathcal E$ be an acceptance event determined by $X$ and $(R_1,\ldots,R_m)$, with $\Pr(\mathcal E)>0$. Since $\mathcal E\in\mathcal K$, the same projection holds under $\Pr(\cdot\mid\mathcal E)$. Therefore
\begin{equation}
 \begin{aligned}
 \E[G_{g,T}\mid\mathcal E]&=\E[G_g\mid\mathcal E],\\
 \Cov(G_g\mid\mathcal E)-\Cov(G_{g,T}\mid\mathcal E)
 &=\E[D_gD_g^\top\mid\mathcal E]\succeq0.
 \end{aligned}
\end{equation}
The reference distribution on both sides is the same accepted-group distribution.
\end{proof}

\paragraph{Retained normalization.}
Measurability of $N_i$ allows its reciprocal to be absorbed into the coefficient. Successful trajectories retain their complete records, so their token counts may depend on the completed trajectory. For failures, a fixed-length protocol or a count determined by the stopped record and outcome satisfies the same requirement. A future-dependent failure count introduces additional information, under which the removed score would have to be centered again. The theorem therefore uses the actual token count known in the retained record.

\paragraph{Binary coefficients and action credit.}
The relation to local action credit can also be made explicit. Let $U$ collect sibling outcomes, which are independent of the local trajectory conditional on $X$. Including $U$ in the initial analytical history preserves score centering. For a bounded coefficient $C=c(R,U)$, set $c_0(U)=c(0,U)$ and $\delta c(U)=c(1,U)-c(0,U)$. The binary outcome then gives $C=c_0(U)+\delta c(U)R$. Substituting the action-conditioned advantage and using score centering gives
\begin{equation}
 \E[CZ_k\mid\F_k]
 =\delta c(U)\E[Z_k\mathcal A^\pi(h_k,a_k)\mid\F_k].
 \label{eq:group_credit_bridge}
\end{equation}
Thus the coefficient difference $\delta c(U)$ weights local action credit. This identity uses the unfiltered on-policy law, whereas the accepted-group result above follows from conditioning the established projection. The two steps account separately for coefficient dependence and group selection.

\subsection{Proof of Corollary~\ref{cor:failure_stop}}

\begin{proof}
On $\{T<H\}$, the condition $P_T=0$ and nonnegativity of $R$ imply $R=0$ almost surely. Hence $\chi_T=0$ there. At $T=H$ the suffix is empty. Moreover, successes have $T=H$, so $\widetilde S_T=S_T$. The outcome is already $\F_T$-measurable, and any additionally retained successful history is already complete. Thus $\mathcal K_T=\F_T$ up to null events. Applying Theorem~\ref{thm:group_stop} proves the corollary.
\end{proof}

For comparison, suppose that terminal success is determined by an independent fair coin revealed at $H$. At every earlier history, $P_t=1/2$, while independence and score centering give $\chi_t=0$. Thus zero remaining credit can hold before the outcome is determined.

\subsection{Ordering ideal stopping rules}
\label{app:nested_stop}

\begin{corollary}[Ordering ideal stopping rules]
\label{cor:earliest_stop}
Consider two rules satisfying Theorem~\ref{thm:group_stop}, with the same coefficients and denominators and $T_{i,1}\le T_{i,2}$ for every trajectory. Their group estimators obey
\begin{equation}
 \Cov(G_{g,T_1})\preceq\Cov(G_{g,T_2}).
 \label{eq:nested_stopping}
\end{equation}
\end{corollary}

\begin{proof}
For $T_{i,1}\le T_{i,2}$, a later stopped prefix determines the earlier stopping record. Both rules retain the final outcome and complete successful trajectories. Therefore $\mathcal K^{(i)}_{T_{i,1}}\subseteq\mathcal K^{(i)}_{T_{i,2}}$, and their joint information sets satisfy $\mathcal K_1\subseteq\mathcal K_2$. With the same full estimator and coefficients measurable in both information sets, the tower property gives
\begin{equation}
 \begin{aligned}
 \E[G_{g,T_2}\mid\mathcal K_1]&=G_{g,T_1},\\
 \Cov(G_{g,T_2})-\Cov(G_{g,T_1})
 &=\E[(G_{g,T_2}-G_{g,T_1})(G_{g,T_2}-G_{g,T_1})^\top]\succeq0.
 \end{aligned}
\end{equation}
The same argument holds under a shared acceptance event in $\mathcal K_1$.
\end{proof}

\paragraph{Earliest zero-credit stop.}
For each trajectory, define
\begin{equation}
 \begin{aligned}
 T_*&=\min\bigl(\{t\in\{1,\ldots,H-1\}:\chi_t=0\}\cup\{H\}\bigr),\\
 B_*&=T_*-1.
 \end{aligned}
 \label{eq:earliest_stop}
\end{equation}
Adaptedness of $\chi_t$ makes this first zero, with fallback $H$, a stopping time. It precedes every stop with $T\ge1$ and $\chi_T=0$. Within the class satisfying Theorem~\ref{thm:group_stop} with the same coefficients and denominators, the ordering result therefore favors earlier stops. Whenever the original denominators are measurable in the joint information retained at $T_*$, this rule belongs to that class and attains its minimum covariance. The construction uses exact conditional credit at the fixed collection policy.

\subsection{Proof of Proposition~\ref{prop:retrospective_projection}}
\label{app:retrospective_projection}

All expectations below use the same fixed reference law, which may be the law conditional on group acceptance.

\begin{proof}
Since $G_B$ is $\mathcal K_B$-measurable,
\begin{equation}
 \E[G\mid\mathcal K_B]=G_B+\E[D_B\mid\mathcal K_B]=G_B+\eta_B.
\end{equation}
Taking expectations gives $\E[G_B]-\E[G]=-\E[\eta_B]$, and Jensen's inequality bounds the squared expectation shift. Moreover,
\begin{equation}
 \Cov(G_B,D_B)=\Cov(G_B,\eta_B),
\end{equation}
because $D_B-\eta_B$ has conditional mean zero given $\mathcal K_B$. Expanding $\Cov(G_B+D_B)$ therefore yields
\begin{equation}
 \begin{aligned}
 \Cov(G)-\Cov(G_B)
 ={}&\Cov(D_B)+\Cov(G_B,\eta_B)\\
    &+\Cov(\eta_B,G_B).
 \end{aligned}
 \label{eq:retrospective_covariance}
\end{equation}
The absolute trace of either cross term is at most
$\sqrt{V_B\tr\Cov(\eta_B)}\le\sqrt{V_B\epsilon_B^2}$, proving Equation~\ref{eq:retrospective_bounds}. If $\eta_B=0$, then $\E[D_B]=0$ and the cross terms vanish, giving
\begin{equation}
 \Cov(G)-\Cov(G_B)=\E[D_BD_B^\top]\succeq0.
\end{equation}
Conversely, the projection identity forces $\eta_B=0$, completing the equivalence in Proposition~\ref{prop:retrospective_projection}. The general bound controls total variance through its trace. When the residual vanishes, the matrix covariance ordering above also holds.
\end{proof}

The definition of $\mathcal K_B$ includes the selected boundaries and actual denominators, so $\eta_B$ accounts for the information they convey. Conditioning on a coarser record $\mathcal C\subseteq\mathcal K_B$ gives $\E[D_B\mid\mathcal C]=\E[\eta_B\mid\mathcal C]$. Jensen's inequality therefore yields
\begin{equation}
 \E\|\E[D_B\mid\mathcal C]\|^2\le\epsilon_B^2.
 \label{eq:coarse_residual}
\end{equation}
This lower bound measures the residual component visible through $\mathcal C$, while $\epsilon_B^2$ uses the full retained information. In particular, a positive coarse moment establishes a nonzero projection residual. Finite-model checks verify these identities and bounds across 18 examples, including retrospective selection, outcome filtering, and future-dependent failure denominators.

\subsection{Exact truncation tradeoff}

\begin{proposition}[Exact truncation tradeoff]
\label{prop:mse}
For square-integrable vectors $G,D$, let $G_w=G-D$, $\mu=\E[G]$, $d=\E[D]$, and $\Delta_V=\tr\Cov(G)-\tr\Cov(G_w)$. Averages of $n$ independent identically distributed copies satisfy Equation~\ref{eq:mse}. Strict MSE improvement relative to $\mu$ occurs exactly when
\begin{equation}
 \Delta_V>n\|d\|^2.
 \label{eq:variance_criterion}
\end{equation}
\end{proposition}

\begin{proof}
The expectation of $G_w$ is $\mu-d$. For an average of $n$ independent identically distributed gradient observations,
\begin{align}
 \E\|\overline G_w-\mu\|^2
 &=\|d\|^2+\frac{1}{n}\tr\Cov(G_w),\\
 \E\|\overline G-\mu\|^2
 &=\frac{1}{n}\tr\Cov(G).
\end{align}
Subtracting gives Equation~\ref{eq:mse}, with $\Delta_V=\tr\Cov(G)-\tr\Cov(G_w)$. To relate this to the removed update, define $V_D=\tr\Cov(D)$ and $C_{wD}=\E\langle G_w-\E G_w,D-d\rangle$. The identity $G=G_w+D$ gives
\begin{equation}
 \tr\Cov(G)=\tr\Cov(G_w)+\tr\Cov(D)
 +2\E\langle G_w-\E G_w,D-d\rangle.
\end{equation}
Thus $\Delta_V=V_D+2C_{wD}$, and strict improvement is equivalent to Equation~\ref{eq:variance_criterion}. The cross term records any noise cancellation between retained and removed contributions. Independence is used across the $n$ copies, while the retained and removed parts of each copy may be dependent.
\end{proof}

\paragraph{Applying the projection-residual bounds.}
For the retrospective estimator, set $G_w=G_B$ and $D=D_B$. Substituting $\|d\|^2\le\epsilon_B^2$ and the total-variance bound from Proposition~\ref{prop:retrospective_projection} gives
\begin{equation}
 \begin{aligned}
 &\E\|\overline G_B-\mu\|^2-\E\|\overline G-\mu\|^2\\
 &\qquad\le\epsilon_B^2-\frac{V_D}{n}
             +\frac{2\sqrt{V_B\epsilon_B^2}}{n}.
 \end{aligned}
\end{equation}
A negative right-hand side is sufficient for MSE improvement. This compares the residual magnitude with the removed variance and the batch size.

\paragraph{Connecting the decomposition to clipped PPO.}
For a fixed nonnegative weight $w$ and fixed advantage $A$, positive homogeneity gives
\begin{equation}
 \min\{\rho wA,\operatorname{clip}(\rho)wA\}
 =w\min\{\rho A,\operatorname{clip}(\rho)A\}.
\end{equation}
Consequently, at the same parameters and on the same data, detached masking with the original denominator decomposes the raw surrogate gradient into retained and removed contributions. Consistent subgradients are used at clipping boundaries, and common regularizers belong to the retained reference part. The decomposition therefore applies to paired surrogate gradients evaluated before an optimizer update.

\paragraph{A statistical design consequence.}
Under the ideal projection, $d=0$ and $C_{wD}=0$, so $\Delta_V=V_D\ge0$. For a fixed general rule with $\Delta_V>0$ and $\|d\|>0$, strict MSE improvement occurs when $n<\Delta_V/\|d\|^2$. Thus increasing the batch size reduces the variance benefit relative to the fixed expectation shift. These comparisons use population moments of the fixed estimators, with empirical estimates subject to sampling uncertainty.

\subsection{Approximate stopping and conditional bias}
\label{app:approx_stop}

The zero-credit condition gives an exact projection. To quantify a nearby case, we bound the conditional expectation of the removed suffix at an adapted stop using the remaining outcome uncertainty.

\begin{proposition}[Conditional bias of failure-only suffix removal]
\label{prop:approx_stop}
Under the finite-horizon score assumptions above, let $T\le H$ be a stopping time, $c_0$ a bounded $\F_T$-measurable failure coefficient, and $S_T^+=\sum_{k=T}^{H-1}Z_k$. For $D=(1-R)c_0S_T^+$,
\begin{equation}
 \begin{aligned}
 \E[D\mid\F_T]
 &=-c_0\E[(R-P_T)S_T^+\mid\F_T]\\
 &=-c_0\E\!\left[\sum_{k=T}^{H-1}Z_k\Delta P_k\,\middle|\,\F_T\right].
 \end{aligned}
 \label{eq:approx_stop_identity}
\end{equation}
The conditional norm obeys
\begin{equation}
 \|\E[D\mid\F_T]\|
 \le |c_0|\sqrt{P_T(1-P_T)\E[\|S_T^+\|^2\mid\F_T]}.
 \label{eq:approx_stop_bound}
\end{equation}
If $P_T\le\eta$, $|c_0|\le L$, and $\E[\|S_T^+\|^2\mid\F_T]\le v^2$ on the early-stop event, the conditional expectation shift there is at most $Lv\sqrt{\eta}$, and at $T=H$ it is zero.
\end{proposition}

\begin{proof}
For each possible stopping value $t$, conditional score centering and the tower property imply $\E[\sum_{k=t}^{H-1}Z_k\mid\F_t]=0$. Splitting over $\{T=t\}$ yields $\E[S_T^+\mid\F_T]=0$. Hence
\begin{align}
 \E[D\mid\F_T]
 &=c_0\bigl(\E[S_T^+\mid\F_T]-\E[RS_T^+\mid\F_T]\bigr)\\
 &=-c_0\E[(R-P_T)S_T^+\mid\F_T].
\end{align}
The remaining-credit identity established above gives the second equality in Equation~\ref{eq:approx_stop_identity}. Applying conditional Cauchy--Schwarz to the sum of paired increments bounds the squared norm by
\begin{equation}
 c_0^2
 \E\!\left[\sum_{k=T}^{H-1}\|Z_k\|^2\,\middle|\,\F_T\right]
 \E\!\left[\sum_{k=T}^{H-1}(\Delta P_k)^2\,\middle|\,\F_T\right].
\end{equation}
Conditional orthogonality of the score increments identifies the first expectation as $\E[\|S_T^+\|^2\mid\F_T]$. Proposition~\ref{prop:innovations} identifies the second as $P_T(1-P_T)$, yielding Equation~\ref{eq:approx_stop_bound}. The stated bounds and $P_T(1-P_T)\le P_T$ give the final inequality. The two factors quantify the remaining score variation and outcome uncertainty.
\end{proof}

Since $G_w=G-D$, the unconditional shift is
\begin{equation}
 \E[G_w]-\E[G]
 =\E\!\left[c_0\sum_{k=T}^{H-1}Z_k\Delta P_k\right].
 \label{eq:mean_shift_innovations}
\end{equation}
The coefficient $c_0$ supplies the failure weight and its sign. For a constant baseline $b$, setting $c_0=-b$ multiplies the expected suffix credit by $-b$. The factor $P_T(1-P_T)$ measures residual outcome uncertainty, while the paired increments determine its contribution to the gradient shift.

The stopping time is chosen from the observed prefix, and $(1-R)$ selects removal after the outcome is known. Equation~\ref{eq:approx_stop_identity} explicitly accounts for this failure selection under the on-policy score law. Its use requires an adapted stop and a failure coefficient measurable at that stop.

\subsection{Comparing retrospective and prefix boundaries}
\label{app:hindsight}

A rule based on the observed prefix provides a reference for examining how complete-video selection changes the removed contribution. We compare the two rules under the same fixed policy and group law, using the same GRPO coefficients and original denominators. This comparison characterizes the combined change in available information and boundary-selection procedure.

\paragraph{Prefix comparison rule.}
Evaluate the frozen event model on successive observed prefixes. At each nonterminal time $t$, an error qualifies according to the event-mapping conditions in Section~\ref{sec:method}, with the recovery field omitted. Let $T'$ be the first qualifying time, with fallback $H$, and set $B'=T'-1$. Here the boundary is fixed at the detection time, so it is determined by the available prefix. This choice makes $T'$ a stopping time when decoding randomness is fixed in advance or included in the local filtration. By comparison, the complete-video rule selects the right endpoint of the last unrecovered error.

For each failed trajectory $i$, write $B_i$ and $B_i'$ for the two boundaries. Define the group-level disagreement event $\mathcal E_{\ne}=\{\exists i:R_i=0,\ B_i\ne B_i'\}$. Both rules retain all successful-rollout contributions. Hence, with $D_{B'}=G-G_{B'}$, their removed contributions agree outside $\mathcal E_{\ne}$, and
\begin{equation}
 \delta^2=\E\|D_B-D_{B'}\|^2
 =\E\!\left[\ind_{\mathcal E_{\ne}}\|D_B-D_{B'}\|^2\right].
 \label{eq:hindsight_discrepancy}
\end{equation}
If $\Pr(\mathcal E_{\ne})>0$, this equals $\Pr(\mathcal E_{\ne})\E[\|D_B-D_{B'}\|^2\mid\mathcal E_{\ne}]$. If the event has probability zero, then $\delta=0$.

\begin{proposition}[Hindsight decomposition]
\label{prop:hindsight}
For square-integrable estimators, write $V_D=\tr\Cov(D_B)$, $V_{B'}=\tr\Cov(G_{B'})$, and $V_{D'}=\tr\Cov(D_{B'})$. With $\Delta_V=\tr\Cov(G)-\tr\Cov(G_B)$, the two rules satisfy
\begin{equation}
 \begin{aligned}
 \|\E[D_B]\|&\le\|\E[D_{B'}]\|+\delta,\\
 \Delta_V&\ge V_D-2\bigl|\tr\Cov(G_{B'},D_{B'})\bigr|
          -2\delta\bigl(\sqrt{V_{B'}}+\sqrt{V_{D'}}\bigr)-2\delta^2.
 \end{aligned}
 \label{eq:hindsight_bounds}
\end{equation}
\end{proposition}

\begin{proof}
Write $D'=D_{B'}$ and $\Delta_D=D_B-D'$, so that $D_B=D'+\Delta_D$ and $G_B=G_{B'}-\Delta_D$. The expectation bound follows from $\E[D_B]=\E[D']+\E[\Delta_D]$ and $\|\E[\Delta_D]\|\le\sqrt{\E\|\Delta_D\|^2}=\delta$. For the variance, $\Delta_V=\tr\Cov(D_B)+2\tr\Cov(G_B,D_B)$ and
\begin{equation}
 \Cov(G_B,D_B)=\Cov(G_{B'},D')+\Cov(G_{B'},\Delta_D)-\Cov(\Delta_D,D')-\Cov(\Delta_D).
 \label{eq:hindsight_cross}
\end{equation}
Since $\tr\Cov(\Delta_D)\le\E\|\Delta_D\|^2=\delta^2$, the Cauchy--Schwarz inequality for traces bounds the last three terms by $\delta\sqrt{V_{B'}}$, $\delta\sqrt{V_{D'}}$, and $\delta^2$, which proves Equation~\ref{eq:hindsight_bounds}. All moments use the fixed reference group law, so the same algebra applies when this is an accepted-group law.
\end{proof}

The first bound separates the prefix rule's expectation shift from the additional discrepancy $\delta$. For a local trajectory under the on-policy law, Proposition~\ref{prop:approx_stop} bounds the prefix shift when its failure coefficient is measurable at $T'$. At group level, Equation~\ref{eq:hindsight_bounds} retains $\E[D_{B'}]$ under the chosen reference law, accounting for the actual coefficients and any group selection.

\paragraph{Diagnostic quantities.}
These terms specify a possible paired evaluation on saved rollout groups. Querying the frozen model on successive prefixes gives each $B_i'$, which can be compared with the complete-video boundary $B_i$ on failed trajectories. The fraction of groups with a differing failure boundary estimates $\Pr(\mathcal E_{\ne})$, and the distribution of $B_i-B_i'$ describes timing differences. Paired gradients on disagreeing groups estimate the conditional moment in $\delta^2$, while gradients under the prefix rule estimate $\tr\Cov(G_{B'},D_{B'})$. Finite models with known disagreement events verify the decomposition, including the zero-disagreement case.

\subsection{Partial suffix removal}
\label{app:shrinkage}

The same MSE decomposition also permits a theoretical comparison of removal strengths at a fixed boundary. Let $D$ be the removed suffix contribution and $w_{\rm prefix}$ the corresponding binary prefix mask on failures. A scalar $\lambda\in[0,1]$ defines $G_\lambda=G-\lambda D$, with temporal weights on failed trajectories given by
\begin{equation}
 w_\lambda=1-\lambda(1-w_{\rm prefix}).
\end{equation}
The prefix keeps weight one and the suffix gets $1-\lambda$. Successful-rollout loss terms and the original denominator are preserved. This family varies the strength of suffix removal, whereas the uniform-scaling ablation in Section~\ref{sec:experiments} scales the entire failed trajectory.

Set $\mu=\E[G]$, $d=\E[D]$, $V_D=\tr\Cov(D)$, and $C_{GD}=\E\langle G-\mu,D-d\rangle$. The same variance expansion gives
\begin{equation}
 \begin{aligned}
 \E\|\overline G_\lambda-\mu\|^2-\E\|\overline G-\mu\|^2
 &=\lambda^2\|d\|^2+\frac{\lambda^2V_D-2\lambda C_{GD}}{n},\\
 \lambda^*&=\operatorname{clip}_{[0,1]}
 \frac{C_{GD}}{V_D+n\|d\|^2}.
 \end{aligned}
 \label{eq:partial_suffix_risk}
\end{equation}
The second line minimizes the quadratic in the first when its denominator is positive. If the denominator is zero, then $D=0$ almost surely and all strengths are equivalent. For a fixed joint distribution of $G$ and $D$, increasing $n$ makes the optimal removal strength $\lambda^*$ nonincreasing.

Under the exact projection, $D$ is centered and uncorrelated with $G-D$, giving $d=0$ and $C_{GD}=V_D$. If $V_D>0$, Equation~\ref{eq:partial_suffix_risk} gives $\lambda^*=1$. Thus complete suffix removal minimizes MSE within this continuous family under the exact projection.

\section{Examples Clarifying the Projection Conditions}
\label{app:counterexample}

The following examples separate remaining conditional credit from success probability, boundary selection, and cancellation across scenes. Each episode has two chunks. Chunk 0 is parameter independent, and chunk 1 samples $a\sim\mathrm{Bernoulli}(p)$ with logit parameter $\theta$, giving scores $Z_0=0$ and $Z_1=a-p$. For a fixed baseline $b\in[0,1]$, the reference estimator is $G=(R-b)Z_1$. Successes retain both chunks. On failure, $T=1$ removes $Z_1$, while $T=2$ retains it, with boundary $B=T-1$. Bars denote averages of $n$ independent episodes, and MSE is measured relative to $\mu=\E[G]$.

\paragraph{Rare success and expectation shift.}
Set $R=a$ and choose the fixed candidate $T=1$. The discarded action determines success, so $\chi_1=p(1-p)>0$ for $0<p<1$. The full and masked estimators satisfy
\begin{equation}
 G=(R-b)(a-p),\qquad \E[G]=p(1-p),
 \label{eq:rare_success_example}
\end{equation}
\begin{equation}
 G_w=\ind\{a=1\}(1-b)(1-p),\qquad
 \E[G_w]=(1-b)p(1-p).
\end{equation}
The expectation shifts by $-bp(1-p)$, a fraction $b$ of the reference expectation. Thus, as $p$ approaches zero, the absolute shift vanishes while its relative size remains fixed. The two possible values of $G$ are $bp$ and $(1-b)(1-p)$, which give
\begin{equation}
 \begin{aligned}
 \Var(G)&=p(1-p)(1-b-p)^2,\\
 \Var(G_w)&=p(1-p)^3(1-b)^2.
 \end{aligned}
\end{equation}
At $b=1/2$ and $p=0.01$, the expectations are $0.0099$ and $0.00495$, and the variances are $0.00237699$ and $0.0024257475$. Both squared bias and variance increase, so masking raises MSE for every $n$. For $n=64$, MSE rises from $0.00003714046875$ to $0.0000624048046875$.

The same construction makes the source of the variance increase explicit at $p=b=1/2$. Here $G=1/4$ on both outcomes, while $G_w=1/4$ on success and zero on failure. Therefore
\begin{equation}
 \begin{aligned}
 \E[G_w]&=\frac18,\\
 \Var(G_w)&=\frac1{64},\\
 \E\|\overline G_w-\mu\|^2&=\frac1{64}+\frac1{64n}.
 \end{aligned}
\end{equation}
Writing $D=G-G_w$ gives $\Var(D)=1/64$ and $\Cov(G_w,D)=-1/64$. The full estimator's zero variance comes from cancellation between retained and removed contributions.

The approximate-stopping bias bound is also attained in this construction. The failure coefficient is $-b$, the remaining score is $a-p$, and $P_1=p$. Its squared bound is therefore $b^2p^2(1-p)^2$, exactly the squared expectation shift. This equality illustrates why a small absolute bias bound can coexist with a fixed relative change.

\paragraph{Zero credit with uncertain outcomes.}
Set $p=b=1/2$, let an independent fair coin determine $R$, and choose $T=1$. Success is independent of the action, so $P_1=1/2$ and $\chi_1=\E[RZ_1\mid\F_1]=0$. The full estimator has expectation zero and variance $1/16$. Masking failures gives $G_w=R(a-1/2)/2$, with expectation zero and variance $1/32$. On failure, the retained information contains the outcome but leaves the action uniform, making the conditional expectation of $G$ zero. On success, the complete record reveals $G$. Thus $G_w$ is the conditional projection even though success remains uncertain at $T=1$.

\paragraph{Future-dependent boundary selection.}
Keep $p=b=1/2$ and the independent outcome coin, but choose $T=1$ when $a=1$ and $T=2$ otherwise. Both deterministic candidates have zero remaining credit. However, $\{T=1\}=\{a=1\}$ depends on an action unobserved at time 1, so $T$ is not a stopping time. Its failure-only mask removes exactly the event $(a,R)=(1,0)$, which has probability $1/4$ and full contribution $-1/4$. The expectation therefore changes from zero to $1/16$. Here the shift arises from selecting on the future action, despite zero credit at each fixed candidate time.

\paragraph{Cancellation across initial scenes.}
Let an observed initial scene $X\in\{0,1\}$ be uniform, with $p=b=1/2$ and $T=1$. Set $R=a$ for $X=0$ and $R=1-a$ for $X=1$. The two conditional credits are $1/4$ and $-1/4$. Within each scene, $G$ is constant at the corresponding value, while $G_w$ retains it only on success. Averaging across scenes gives expectation zero for both estimators and variances $1/16$ and $1/32$, respectively. Yet on failure, the retained scene identifies $G=\pm1/4$, whereas $G_w=0$. The equal global expectations therefore result from cancellation across scenes rather than conditional projection.

\paragraph{Finite enumeration check.}
Exact rational enumeration verifies the expectations, variances, and MSEs in these examples. Further checks cover local and group projections, nested stops, variable denominators, outcome filtering, approximate bias, and partial suffix removal. These finite calculations complement the analytic proofs.

\section{Policy Training and Evaluation}
\label{app:config}

Experiments are conducted in RoboTwin 2.0~\citep{chen2025robotwin}, a SAPIEN-based simulation benchmark for bimanual manipulation. The seven tasks cover tool use, handover, lifting, placement, and stacking in randomized scenes. A rollout ends when the task-specific success condition is satisfied or its execution budget is exhausted. The simulator assigns a binary terminal reward $R\in\{0,1\}$ from this outcome.

\subsection{Training and evaluation settings}
\label{app:grpo8}

The policy uses the action-token OpenVLA-OFT implementation from SimpleVLA-RL~\citep{kim2025oft,li2025simplevla}. For each task, the RL methods start from the same SFT checkpoint. Table~\ref{tab:policy_training_config} summarizes the rollout collection, optimization, and evaluation settings. One policy update collects 64 accepted groups (512 trajectories) and makes one PPO pass in minibatches of 128, giving four optimizer steps.

\input{tables/policy_training_config}

Table~\ref{tab:grpo8_baseline} lists task-specific episode budgets and Full GRPO success at every scheduled evaluation from update 0 to 40. Its initial and update-40 evaluations supply the SFT and Full GRPO columns in Table~\ref{tab:eight_task_policy}.

\input{tables/grpo_baseline}

\subsection{Feedback accounting and normalization}
\label{app:feedback_normalization}

The collection policy applies temperature 1.6 on the action-token support, and token likelihoods use the same distribution. In Equation~\ref{eq:group_advantage}, $s_g$ is the sample standard deviation of group rewards and $\epsilon_{\mathrm{num}}=10^{-6}$. Each chunk score $Z_{i,k}$ sums its constituent token scores. At the collection parameters, the unclipped full-feedback objective therefore contributes $A_i Z_{i,k}/N_i$ per chunk, matching the reference estimator's weighting.

The loss averages action-token terms within each trajectory before gradient accumulation, using the original valid-token count $N_i$ as its denominator. This implements the normalization in Equation~\ref{eq:ppo}. To quantify the feedback retained by a mask $w$, define its pre-clipping negative coefficient mass as
\begin{equation}
 D_-(w)=\frac{1}{M}\sum_{i:R_i=0}\frac{|A_i|}{N_i}
 \sum_{u\in\mathcal V_i}w_{i,k(u)}.
 \label{eq:mass}
\end{equation}
This quantity combines the retained token weights with each failed rollout's advantage and original denominator. The ratio $D_-(w)/D_-(1)$ measures the retained fraction of feedback weight used in the ablation analysis. For example, retaining three of eight equally populated chunks leaves $3/8$ of a failed trajectory's original coefficient mass. Across trajectories, the aggregate fraction also depends on their advantage weights.

Theorem~\ref{thm:group_stop} additionally requires $N_i$ to be measurable from the retained information. In the implementation, the valid-token mask is determined by the executed control steps, and $N_i$ is its sum. Feedback truncation sets suffix advantages to zero while retaining this mask and its denominator. A fully valid eight-chunk Hammer failure has $N_i=8\times25\times14=2800$. Earlier termination reduces the valid-token count before feedback truncation is applied.

\input{figures/vlm_example_style}

\section{VLM Event Model: Mapping, Corpus, and Localization}
\label{app:labels}
\label{app:vlm}

\subsection{VLM inputs, prompting, and boundary mapping}
\label{app:event_rule}

\paragraph{Input construction.}
Each request contains observations from one completed rollout. Table~\ref{tab:vlm_inputs} lists their sources and representation. The three camera streams are supplied as separate video inputs, while task context, robot observations, and temporal alignment enter the user message as text. Terminal rewards are used by the training loop to select failed rollouts for inference.

\input{tables/vlm_inputs}

Each full action chunk contains 25 control steps. Frame 0 precedes chunk 0, and a frame with within-chunk fraction 1 records the end of that chunk. These indices locate events on the policy timeline.

\paragraph{Prompt structure.}
A shared template combines the task instruction and a fixed task-specific clarification with the executed chunk count, initial action budget, and synchronized observations. The clarification specifies the intended manipulation, such as stacking the bowls and releasing the moved bowl. The system instruction and an abridged version of the user instruction are shown below.

\begin{center}
\vlmexamplecard{vlmprompt}{VLM event-localization prompt}{%
\textbf{System instruction.} Review robot video evidence. Return the requested concise JSON.

\medskip
\textbf{User instruction (abridged).} Review the synchronized camera views and robot observations for the given task. Track objects across views and check whether apparent errors are followed by regrasping or task completion. Identify the last visible error sequence without subsequent recovery and localize its onset using the recorded frame-to-chunk mapping. Return a chunk interval when timing is ambiguous and a null interval when no specific error can be localized.
}{%
\textbf{Required response.} Return a compact JSON response containing the observed outcome, grasp and handover status, error type, error interval, recovery status, final holder, up to six supporting frame indices, and a brief factual summary.
}
\end{center}

Frozen-model inference uses greedy decoding with at most 512 generated tokens. The resulting event record supplies the outcome, interval, and recovery information used by the boundary rule.

\paragraph{Boundary mapping and fallbacks.}
The mapping in Equations~\ref{eq:events} and~\ref{eq:event_boundary} is applied in the following order.

\begin{enumerate}
 \setlength{\itemsep}{2pt}
 \setlength{\parskip}{0pt}
 \item \textbf{Select rollouts.} Terminal environment rewards determine which rollouts require inference. Failed rollouts are passed to the event model, while successful rollouts bypass it and retain all their loss terms.
 \item \textbf{Validate responses.} The mapper checks the response schema and any non-null interval. A valid interval contains two integer chunk indices satisfying $0\le\ell_i\le u_i<K_i$. Malformed responses, identity mismatches, and failed inference requests raise implementation errors rather than becoming no-event predictions.
 \item \textbf{Apply semantic fallbacks.} Full feedback is retained for a predicted success or uncertain outcome, no clear error, an uncertain error type, or a missing localization interval. It is also retained when subsequent recovery from the localized error is reported or uncertain.
 \item \textbf{Set the boundary.} For a valid failure prediction with a localized error and no observed recovery, the mapper returns the upper interval endpoint as $B_i=u_i$.
\end{enumerate}

Supporting frame identifiers provide evidence for inspection. The validated event fields determine the feedback boundary.

\subsection{Annotated video corpus}

The annotated corpus contains 1,151 rollout videos, comprising 1,024 environment-failed trajectories and 127 successful references. Table~\ref{tab:vlm_corpus_split} summarizes the training, validation, and test partitions.

\input{tables/vlm_corpus_split}

Task--scene identity determines the split even when the same scene appears at several policy checkpoints. The full collection has 1,045 unique task--scene identities with zero overlap across partitions. The corpus also excludes the held-out scene seeds used for policy evaluation. Gradient fitting uses the training partition, and checkpoint ranking uses validation metrics.

Each video has structured event annotations recording observed outcome, grasp establishment, handover completion, error type, error localization interval, observed recovery, final holder, evidence frames, and a summary. The error interval expresses uncertainty about the onset of the last visible error without observed recovery. A null interval is allowed when no error can be localized, and unknown fields are preserved. Figure~\ref{fig:vlm_dataset_example} pairs a video excerpt with its annotation. In this example, the initial pickup difficulty is resolved, and the interval identifies the later placement error in chunk 12.

\input{figures/vlm_dataset_example}

Environment outcomes are stored separately from versioned visible-event labels. The collection's sampling design and recording length can reveal outcome information. The complete corpus manifest contains 1,024 predicted visible failures, 71 visible successes, and 56 uncertain outcomes, with 1,011 non-null intervals and 140 null intervals. These label counts cover all 1,151 videos.

The event model trained on this corpus is documented in the next subsection, and the policy runs that use it in Appendix~\ref{app:epoch8_policy}.

\subsection{Event model training}
\label{app:vlm8_training}

\paragraph{Adapter configuration.}
One adapter is shared across all annotated tasks. Table~\ref{tab:vlm_training_config} summarizes its LoRA and optimization settings, which remain fixed throughout training. The response loss follows Equation~\ref{eq:student}, with the data partitions in Table~\ref{tab:vlm_corpus_split}. Training and localization evaluation use the chat template with thinking disabled.

\input{tables/vlm_training_config}

\paragraph{Checkpoint selection.}
Checkpoints are ranked by validation task-macro IoU, with lower response NLL and then earlier update as tie-breakers. Both the highest validation IoU ($78.2\%$) and lowest NLL occur at update 304 (epoch 8), which is retained. A fresh patience counter starts at update 361, with validation every 19 updates (half an epoch). Four checks without an IoU improvement of at least $0.005$ over the last reset anchor trigger stopping at update 437 (11.5 epochs). This patience rule controls training duration separately from checkpoint ranking.

\paragraph{Localization evaluation.}
Test diagnostics also informed training duration. Table~\ref{tab:vlm8_localization} reports agreement with versioned reference intervals on 220 localizable, unrecovered test failures from the seven tasks in the policy comparison. Eligibility is determined from reference annotations independently of model predictions. The reported checkpoint is selected by validation IoU.

\input{tables/event_localization}

\paragraph{Qualitative localization.}
Figure~\ref{fig:vlm_localization_example} shows a failed Handover in which the selected model localizes the transfer error to chunks $[4,5]$, matching the reference annotation. Head-camera frames place this interval between the initial grasp and later motion without recovery.

\input{figures/vlm_localization_example}

\paragraph{Localization across checkpoints.}
Figure~\ref{fig:vlm_checkpoint_intervals} follows the prediction for the same rollout across checkpoints. The base model predicts $[6,8]$, and epoch 3 predicts $[5,5]$. At epochs 6 and 8, both endpoints match the reference interval $[4,5]$.

\input{figures/vlm_checkpoint_intervals}

\section{Policy Results and Feedback Ablations}
\label{app:diagnostics}
\label{app:curves}
\label{app:ablations}
\label{app:cases}

This appendix provides scheduled policy evaluations, feedback statistics, and ablation details.

\subsection{Seven-task policy results}
\label{app:epoch8_policy}

The frozen event model is the validation-selected epoch-8 adapter (update 304) of Appendix~\ref{app:vlm8_training}. Table~\ref{tab:eight_task_policy} reports the seven-task policy comparisons at update 40 and their best scheduled scores.

\paragraph{Training configuration.}
All methods follow the training and evaluation settings of Appendix~\ref{app:grpo8}. Microbatches for actor updates and log-probability computation contain one trajectory per rank, and gradient accumulation preserves the global PPO minibatch of 128 trajectories.

The fixed-endpoint results use the scheduled evaluation after 40 updates, with 128 seen-seed and 128 held-out-seed episodes and greedy decoding. The Best columns additionally report each run's maximum pooled success over the same eight evaluations at updates $5,10,\ldots,40$, and their means average the per-task maxima. The SFT column uses the Full-GRPO run's initial evaluation. Both methods share the task checkpoint and evaluation seed splits. Separate simulator executions retain some variability under greedy action decoding, so their initial success rates can differ. Table~\ref{tab:epoch8_trajectories} provides all scheduled scores across the seven tasks. The update budget fixes the number of accepted training groups. Adaptive refill can require different numbers of generated trajectories to reach this budget.

\input{tables/policy_learning_curves}

\paragraph{Feedback usage.}
\label{app:feedback_accounting}
Table~\ref{tab:feedback_usage} summarizes the feedback retained by the event masks.

\input{tables/feedback_usage}

The statistics pool event records and loss-mask coefficients over all 40 updates of each MART run, using the same frozen adapter and preserving success weights and valid-token denominators. Retained coefficient mass in Table~\ref{tab:feedback_usage} is the ratio of summed post-mask to summed pre-mask mass. A localized event at the final chunk retains full feedback without a semantic fallback.

Success rates are computed from integer episode counts, with pooled success calculated over all 256 episodes across both seed splits.

\subsection{Feedback ablations}
\label{app:two_task_ablation}

The ablations cover Hammer and Handover, selected from the seven-task policy comparisons. For each task, statistics are pooled over all 40 MART updates and frozen before ablation training. Uniform scaling divides the retained negative coefficient mass summed over these updates by the corresponding sum of original mass, using Equation~\ref{eq:mass}. This gives $\gamma=0.4920678764$ for Hammer and $0.3268904598$ for Handover. Random prefix draws independently from the empirical boundary histogram of accepted failures, including full-feedback fallbacks. The ablation RNG is separate from policy sampling. Each sampled boundary is a zero-based chunk index, and its prefix is intersected with the new rollout's valid-token range. Both ablations preserve success weights, original token denominators, and terminal rewards, and run without VLM inference.

All ablations use the same task SFT checkpoints, 64 accepted groups of eight rollouts per update, global PPO minibatch 128, and the update-40 evaluation endpoint with 128 seen-seed and 128 held-out-seed episodes. Both ablations reuse frozen statistics, with uniform scaling matching the aggregate coefficient ratio and random prefix matching the boundary distribution. Their realized coefficient-mass ratios are measured on each run's own rollout groups.

At update 40, uniform scaling scores $100/128$ on seen-seed and $99/128$ on held-out-seed scenes for Hammer ($199/256$ combined), and $50/128$ and $54/128$ for Handover ($104/256$ combined). Random prefix scores $96/128$ and $88/128$ for Hammer ($184/256$ combined), and $46/128$ and $60/128$ for Handover ($106/256$ combined), in the same split order.

Figure~\ref{fig:feedback_ablations} compares the endpoints alongside each run's observed retained negative coefficient mass, pooled over all 40 updates using Equation~\ref{eq:mass}. Under these similar feedback-weight budgets, event-guided truncation achieves higher success than both ablations on both tasks.

\input{figures/feedback_ablations}

%% file: tables/policy_training_config.tex
\begin{table}[!ht]
 \caption{\textbf{Policy training and evaluation settings.} Configuration for the reported policy comparisons and feedback ablations.}
 \label{tab:policy_training_config}
 \centering\small
 \renewcommand{\arraystretch}{1.08}
 \begin{tabularx}{\linewidth}{@{}p{0.36\linewidth}X@{}}
 \toprule
 Parameter & Setting\\
 \midrule
 \multicolumn{2}{@{}l}{\textit{Policy and rollout collection}}\\
 Server capacity & 8 NVIDIA H100 GPUs\\
 Initialization & Task-specific SFT checkpoint\\
 Actor observations & Head-camera view and proprioception\\
 Environment & Randomized RoboTwin 2.0 scenes, seen instructions\\
 Training initial conditions & 1,000 per task\\
 Action chunk & 25 control steps, 14 action dimensions per step\\
 Accepted rollout batch & 64 groups of 8 trajectories per update\\
 Group filtering & Success fraction in $[0.1,0.9]$, with adaptive refill\\
 Sampling temperature & 1.6\\
 \midrule
 \multicolumn{2}{@{}l}{\textit{Policy optimization}}\\
 Optimizer & AdamW, $\beta_1=0.9$, $\beta_2=0.999$\\
 Learning rate & $5\times10^{-6}$\\
 Weight decay & $0.01$\\
 Actor LoRA & Disabled\\
 Gradient clipping & Global norm capped at 1\\
 PPO passes per rollout batch & 1\\
 Global PPO minibatch & 128 trajectories\\
 Microbatch per GPU & 1 trajectory\\
 PPO clipping & $\varepsilon_-=0.2$, $\varepsilon_+=0.28$\\
 KL / entropy coefficients & 0 / 0\\
 Loss normalization & Original valid-token count of each trajectory\\
 Training budget & 40 policy updates\\
 \midrule
 \multicolumn{2}{@{}l}{\textit{Policy evaluation}}\\
 Decoding & Greedy\\
 Schedule & Updates $0,5,\ldots,40$\\
 Episodes per evaluation & 128 seen-seed and 128 held-out-seed\\
 Scene splits & Training-seed and held-out-seed scenes, with the same task configuration\\
 \bottomrule
 \end{tabularx}
\end{table}

%% file: tables/grpo_baseline.tex
\begin{table}[!ht]
 \caption{\textbf{Task budgets and Full GRPO learning progress.} Success (\%) pooled over 128 seen-seed and 128 held-out-seed episodes at each update. Update 0 is the SFT evaluation. Episode budgets are maxima, with 25 action steps per chunk.}
 \label{tab:grpo8_baseline}
 \centering\small
 \setlength{\tabcolsep}{1.8pt}
 \renewcommand{\arraystretch}{1.08}
 \begin{tabular*}{\linewidth}{@{\extracolsep{\fill}}lrrrrrrrrrrr@{}}
 \toprule
 Task & \multicolumn{2}{c}{Episode budget} & \multicolumn{9}{c}{Full GRPO success (\%) at update}\\
 \cmidrule(lr){2-3}\cmidrule(l){4-12}
 & \shortstack{Max. action\\steps} & \shortstack{Max.\\chunks} & 0 & 5 & 10 & 15 & 20 & 25 & 30 & 35 & 40\\
 \midrule
 Hammer & 200 & 8 & 30.5 & 35.5 & 48.0 & 46.9 & 53.9 & 63.7 & 65.2 & 74.6 & 74.2\\
 Handover & 400 & 16 & 21.1 & 24.6 & 32.0 & 30.5 & 31.2 & 34.4 & 35.5 & 33.6 & 34.8\\
 Lift pot & 200 & 8 & 5.5 & 20.7 & 29.3 & 35.9 & 44.5 & 46.1 & 52.0 & 54.3 & 58.2\\
 Dual bottles & 100 & 4 & 20.7 & 21.5 & 32.0 & 30.1 & 35.9 & 44.1 & 43.0 & 41.4 & 44.9\\
 Container on Plate & 150 & 6 & 87.1 & 91.4 & 91.4 & 93.8 & 92.6 & 93.0 & 93.8 & 94.9 & 93.8\\
 Empty cup & 200 & 8 & 80.9 & 82.8 & 86.7 & 83.6 & 85.9 & 90.6 & 89.8 & 89.5 & 89.1\\
 Stack bowls & 400 & 16 & 39.1 & 46.1 & 47.7 & 56.2 & 57.0 & 56.2 & 60.2 & 63.3 & 60.2\\
 \bottomrule
 \end{tabular*}
\end{table}

%% file: figures/vlm_example_style.tex
\definecolor{vlmprompt}{RGB}{38,77,112}
\definecolor{vlmannotation}{RGB}{35,112,87}
\definecolor{vlmerror}{RGB}{181,83,67}
\definecolor{vlmmuted}{RGB}{86,101,116}
\definecolor{vlmframe}{RGB}{190,200,209}

\newcommand{\vlmexamplecard}[4]{%
 \begingroup
 \setlength{\fboxsep}{0pt}%
 \setlength{\fboxrule}{0.5pt}%
 \noindent\fcolorbox{#1}{#1!4}{%
  \begin{minipage}{\dimexpr\linewidth-2\fboxrule\relax}
   \setlength{\fboxsep}{6pt}%
   \noindent\colorbox{#1}{%
    \parbox{\dimexpr\linewidth-2\fboxsep\relax}{\small\color{white}\bfseries #2}}%
   \par\nointerlineskip
   \setlength{\fboxsep}{8pt}%
   \noindent\colorbox{#1!4}{%
    \begin{minipage}{\dimexpr\linewidth-2\fboxsep\relax}
     \small\raggedright #3
    \end{minipage}}%
   \par\nointerlineskip
   \noindent\begin{tikzpicture}
    \draw[color=#1!65,line width=0.4pt,dash pattern=on 2.2pt off 2.2pt]
     (0,0)--(\linewidth,0);
   \end{tikzpicture}%
   \par\nointerlineskip
   \noindent\colorbox{#1!4}{%
    \begin{minipage}{\dimexpr\linewidth-2\fboxsep\relax}
     \small\raggedright #4
    \end{minipage}}%
  \end{minipage}}%
 \endgroup
}

\newcommand{\vlmexampleframe}[5]{%
 \begin{minipage}[t]{0.238\linewidth}
  \centering
  {\scriptsize\color{vlmmuted}Frame #1 / Chunk #2\par}
  \vspace{3pt}
  \begingroup
  \setlength{\fboxsep}{0pt}%
  \setlength{\fboxrule}{0.5pt}%
  \fcolorbox{#5}{white}{%
   \includegraphics[width=\dimexpr\linewidth-2\fboxrule\relax]{#3}}%
  \endgroup
  \par\vspace{3pt}
  {\small\color{#5}#4\par}
 \end{minipage}%
}

%% file: tables/vlm_inputs.tex
\begin{table}[!ht]
 \caption{\textbf{VLM inputs.} Sources and representation for a single rollout request.}
 \label{tab:vlm_inputs}
 \centering\small
 \renewcommand{\arraystretch}{1.12}
 \begin{tabularx}{\linewidth}{@{}>{\raggedright\arraybackslash}p{0.17\linewidth}>{\raggedright\arraybackslash}p{0.23\linewidth}>{\raggedright\arraybackslash}X@{}}
 \toprule
 Input & Source & Representation\\
 \midrule
 Task context & Task configuration and rollout metadata & Natural-language instruction, task clarification, executed chunk count, and initial action budget.\\
 Camera streams & Head and two wrist cameras & Three synchronized video sequences. Native $320\times240$ RGB frames are extended to $320\times288$ with camera, frame, and chunk labels. Sampling uses the initial frame and up to six frames per chunk.\\
 Robot state & Robot observation records & Gripper commands, measured openings, and end-effector poses at each sampled frame. Joint positions and velocities are included at the initial frame and chunk ends. Values are serialized as JSON text.\\
 Temporal alignment & Rollout recorder & Frame and chunk indices, within-chunk progress, and simulation timestamps.\\
 \bottomrule
 \end{tabularx}
\end{table}

%% file: tables/vlm_corpus_split.tex
\begin{table}[!ht]
 \caption{\textbf{Annotated video corpus.} Video counts by environment outcome and data split.}
 \label{tab:vlm_corpus_split}
 \centering\small
 \renewcommand{\arraystretch}{1.08}
 \begin{tabular}{@{}lrrrr@{}}
 \toprule
 Video type & Training & Validation & Test & Total\\
 \midrule
 Failed rollouts & 640 & 128 & 256 & 1,024\\
 Successful references & 118 & 5 & 4 & 127\\
 \midrule
 Total & 758 & 133 & 260 & 1,151\\
 \bottomrule
 \end{tabular}
\end{table}

%% file: figures/vlm_dataset_example.tex
\begin{figure}[!ht]
 \centering
 \begin{minipage}{\linewidth}
  \small
  \noindent\textbf{(a) Video excerpt and instruction}\hfill\textcolor{vlmmuted}{Stack bowls}
  \par\vspace{5pt}
  \noindent``Stack the ceramic bowl carefully over the blue bowl with slightly rounded base.''
  \par\vspace{9pt}
  \noindent
  \vlmexampleframe{36}{5}{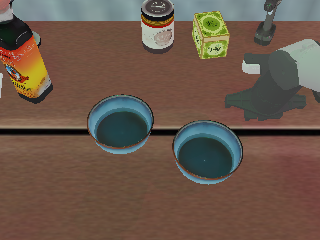}{Pickup attempt}{vlmmuted}\hfill
  \vlmexampleframe{69}{11}{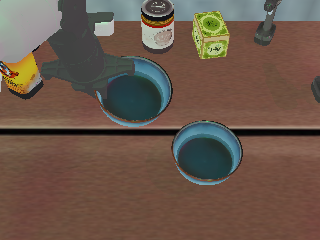}{Bowl grasped}{vlmprompt}\hfill
  \vlmexampleframe{78}{12}{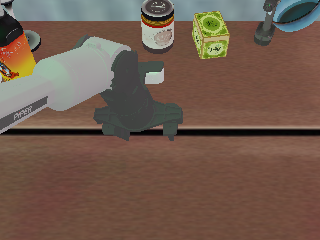}{Placement error}{vlmerror}\hfill
  \vlmexampleframe{96}{15}{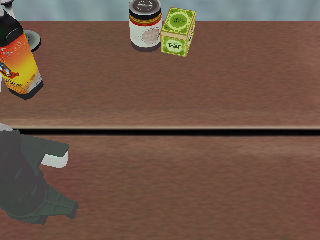}{Rollout end}{vlmmuted}
  \par\vspace{10pt}
  \vlmexamplecard{vlmprompt}{(b) Event annotation}{%
   \renewcommand{\arraystretch}{1.15}%
   \begin{tabularx}{\linewidth}{@{}lXlX@{}}
    \textbf{Observed outcome} & Failure & \textbf{Grasp established} & Yes\\
    \textbf{Error type} & Placement failure & \textbf{Later recovery} & Not observed\\
    \textbf{Error interval} & \textcolor{vlmerror}{Chunks [12, 12]} & \textbf{Final holder} & Neither arm
   \end{tabularx}%
  }{%
   \textbf{Summary.} After the initial grasping difficulty is resolved, releasing the bowl during stacking displaces the bowls. No subsequent recovery is observed.
  }
 \end{minipage}
 \caption{\textbf{An annotated rollout.} Selected head-camera frames and event labels from the video corpus. The interval identifies the placement error after the initial grasping difficulty is resolved. Chunk indices are zero-based.}
 \label{fig:vlm_dataset_example}
\end{figure}

%% file: tables/vlm_training_config.tex
\begin{table}[!t]
 \caption{\textbf{LoRA training configuration.} Settings for the shared VLM event adapter.}
 \label{tab:vlm_training_config}
 \centering\small
 \renewcommand{\arraystretch}{1.10}
 \begin{tabularx}{\linewidth}{@{}>{\raggedright\arraybackslash}p{0.32\linewidth}>{\raggedright\arraybackslash}X@{}}
 \toprule
 Parameter & Setting\\
 \midrule
 \multicolumn{2}{@{}l}{\textit{Adapter}}\\
 Adaptation scope & One shared adapter across tasks\\
 Target layers & \texttt{q\_proj}, \texttt{k\_proj}, \texttt{v\_proj}, \texttt{o\_proj}, \texttt{in\_proj\_qkv}, \texttt{in\_proj\_z}, \texttt{in\_proj\_b}, \texttt{in\_proj\_a}, \texttt{out\_proj}\\
 Rank & 8\\
 LoRA scaling parameter & 16\\
 Dropout & 0\\
 Frozen components & Base model, vision encoder, and projector\\
 \midrule
 \multicolumn{2}{@{}l}{\textit{Optimization}}\\
 Optimizer & AdamW\\
 Learning rate & $2\times10^{-5}$\\
 Weight decay & 0\\
 Gradient clipping & Global norm capped at 1\\
 Effective batch size & 20 trajectories\\
 Compute precision & Bfloat16\\
 Loss normalization & Mean over response tokens within each trajectory, then over trajectories in the batch\\
 \bottomrule
 \end{tabularx}
\end{table}

%% file: tables/event_localization.tex
\begin{table}[!htb]
 \caption{\textbf{Seven-task event localization.} Interval IoU (\%) against versioned reference intervals on 220 localizable, unrecovered test failures from the tasks in the policy comparison. Update 304 is the validation-selected checkpoint. Scores measure label agreement under the evaluation protocol in Appendix~\ref{app:vlm8_training}.}
 \label{tab:vlm8_localization}
 \centering\small
 \begin{tabular}{lrrrrr}
 \toprule
 Task & $n$ & Base & Epoch 3 & Epoch 6 & Epoch 8\\
 \midrule
 beat\_block\_hammer & 32 & 34.0 & 21.9 & 57.8 & 60.9\\
 handover\_block & 32 & 36.5 & 43.8 & 70.3 & 73.4\\
 lift\_pot & 32 & 35.7 & 60.9 & 77.6 & 80.7\\
 pick\_dual\_bottles & 32 & 33.3 & 70.3 & 75.0 & 73.4\\
 place\_container\_plate & 31 & 31.7 & 58.1 & 79.0 & 77.4\\
 place\_empty\_cup & 32 & 24.8 & 53.1 & 59.4 & 72.9\\
 stack\_bowls\_two & 29 & 3.8 & 48.3 & 55.2 & 65.5\\
 \midrule
 Task-macro mean & 220 & 28.5 & 50.9 & 67.8 & 72.1\\
 \bottomrule
 \end{tabular}
\end{table}

%% file: figures/vlm_localization_example.tex
\begin{figure}[!htb]
 \centering
 \includegraphics[width=\linewidth]{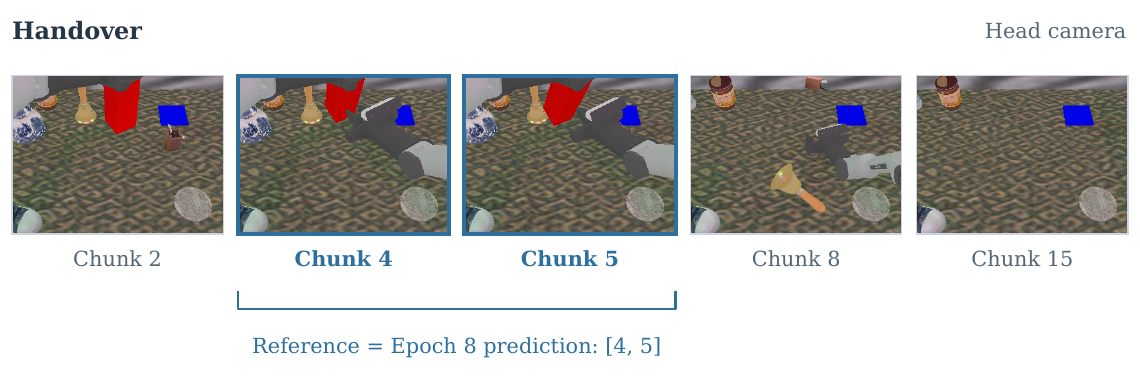}
 \caption{\textbf{Event localization on a test rollout.} Head-camera frames of a failed Handover. Blue outlines mark the interval identified by both the reference annotation and the selected model. Chunk indices are zero-based.}
 \label{fig:vlm_localization_example}
\end{figure}

%% file: figures/vlm_checkpoint_intervals.tex
\begin{figure}[t]
 \centering
 \includegraphics[width=0.60\linewidth]{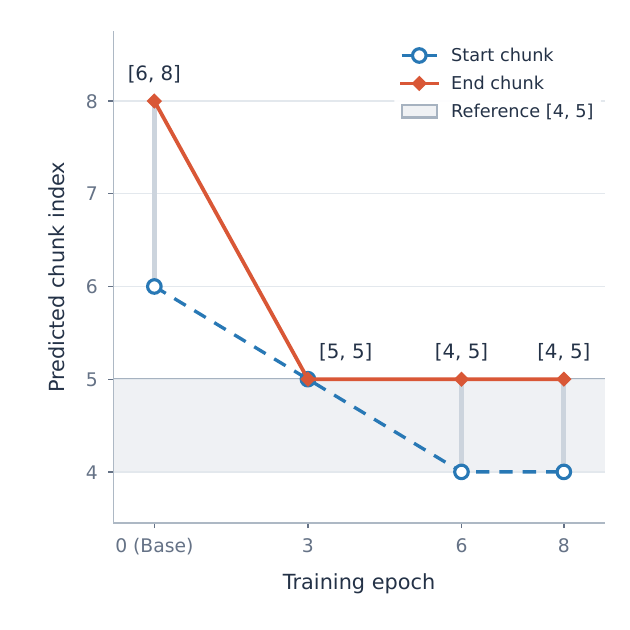}
 \caption{\textbf{Predicted intervals across checkpoints.} Interval endpoints for the Handover rollout in Figure~\ref{fig:vlm_localization_example}. The shaded band marks the reference interval $[4,5]$.}
 \label{fig:vlm_checkpoint_intervals}
\end{figure}

%% file: tables/policy_learning_curves.tex
\begin{table}[ht]
 \caption{\textbf{Scheduled policy evaluations.} Combined success (\%) at each scheduled update. Table~\ref{tab:eight_task_policy} reports update 40 and the per-run maximum over updates $5,10,\ldots,40$. SFT evaluations at update 0 use the same task checkpoint in separate simulator executions. The dash denotes the unavailable initial Handover MART evaluation.}
 \label{tab:epoch8_trajectories}
 \centering\small
 \setlength{\tabcolsep}{1.6pt}
 \begin{tabular*}{\linewidth}{@{\extracolsep{\fill}}l*{7}{rr}@{}}
 \toprule
 & \multicolumn{2}{c}{Hammer} & \multicolumn{2}{c}{Handover}
 & \multicolumn{2}{c}{Lift pot} & \multicolumn{2}{c}{Dual bottles}
 & \multicolumn{2}{c}{Container} & \multicolumn{2}{c}{Empty cup}
 & \multicolumn{2}{c}{Stack bowls}\\
 \cmidrule(lr){2-3}\cmidrule(lr){4-5}\cmidrule(lr){6-7}\cmidrule(lr){8-9}
 \cmidrule(lr){10-11}\cmidrule(lr){12-13}\cmidrule(lr){14-15}
 Update & Full & MART & Full & MART & Full & MART & Full & MART
 & Full & MART & Full & MART & Full & MART\\
 \midrule
 0 & 30.5 & 30.9 & 21.1 & \textemdash & 5.5 & 5.9 & 20.7 & 20.3 & 87.1 & 86.3 & 80.9 & 80.9 & 39.1 & 37.1\\
 5 & 35.5 & 42.6 & 24.6 & 25.8 & 20.7 & 24.6 & 21.5 & 23.0 & 91.4 & 88.7 & 82.8 & 84.4 & 46.1 & 40.2\\
 10 & 48.0 & 49.6 & 32.0 & 25.4 & 29.3 & 31.2 & 32.0 & 25.0 & 91.4 & 91.4 & 86.7 & 87.5 & 47.7 & 52.3\\
 15 & 46.9 & 63.3 & 30.5 & 39.1 & 35.9 & 38.3 & 30.1 & 34.4 & 93.8 & 94.9 & 83.6 & 89.5 & 56.2 & 56.2\\
 20 & 53.9 & 69.1 & 31.2 & 35.2 & 44.5 & 42.2 & 35.9 & 37.1 & 92.6 & 94.1 & 85.9 & 89.5 & 57.0 & 59.0\\
 25 & 63.7 & 77.0 & 34.4 & 37.1 & 46.1 & 48.8 & 44.1 & 37.9 & 93.0 & 94.1 & 90.6 & 87.5 & 56.2 & 60.9\\
 30 & 65.2 & 78.1 & 35.5 & 42.6 & 52.0 & 58.2 & 43.0 & 41.4 & 93.8 & 92.6 & 89.8 & 88.3 & 60.2 & 59.8\\
 35 & 74.6 & 80.5 & 33.6 & 44.5 & 54.3 & 58.6 & 41.4 & 51.2 & 94.9 & 94.1 & 89.5 & 87.1 & 63.3 & 67.2\\
 40 & 74.2 & \textbf{81.2} & 34.8 & \textbf{43.8} & 58.2 & \textbf{58.6} & 44.9 & \textbf{48.8} & \textbf{93.8} & 91.8 & 89.1 & \textbf{91.4} & 60.2 & \textbf{64.8}\\
 \bottomrule
 \end{tabular*}
\end{table}

%% file: tables/feedback_usage.tex
\begin{table}[t]
 \caption{\textbf{What the event mask removes.} Pooled over all 40 updates of each MART run. $N_-$ counts accepted failed trajectories. Shortened denotes a boundary before the final chunk. Retained mass is the ratio of summed negative coefficient mass after and before masking. Fallbacks retain full feedback. These statistics quantify loss weights.}
 \label{tab:feedback_usage}
 \centering\small
 \begin{tabular*}{\linewidth}{@{\extracolsep{\fill}}lrrrr@{}}
 \toprule
 Task & $N_-$ & Shortened (\%) & Retained mass (\%) & Fallback (\%)\\
 \midrule
Hammer & 8,304 & 90.6 & 49.2 & 8.3\\
Handover & 12,955 & 99.4 & 32.7 & 0.5\\
Lift pot & 11,577 & 99.7 & 47.4 & 0.2\\
Dual bottles & 13,296 & 97.9 & 59.3 & 1.9\\
Container on Plate & 4,364 & 97.8 & 51.8 & 2.0\\
Empty cup & 5,612 & 94.7 & 53.6 & 4.8\\
Stack bowls & 10,176 & 96.4 & 37.9 & 2.7\\
 \bottomrule
 \end{tabular*}
\end{table}

%% file: figures/feedback_ablations.tex
\begin{figure}[!ht]
 \centering
 \includegraphics[width=0.90\linewidth]{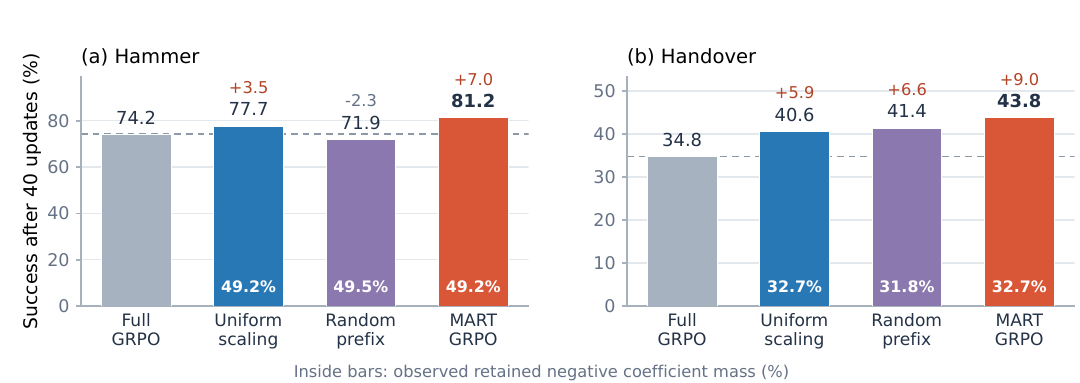}
 \caption{\textbf{Feedback ablations.} Success at update 40, with dashed Full GRPO references and percentage-point changes relative to Full GRPO above the bars. Percentages inside the bars show each run's observed retained negative coefficient mass, pooled over all 40 updates.}
 \label{fig:feedback_ablations}
\end{figure}